\documentclass{vldb}

\newif\iffullversion
\fullversiontrue
\iffullversion
\makeatletter
\def\@copyrightspace{\relax}
\makeatother
\makeatletter
\def\@mkbibcitation{\relax}
\makeatother
\fi
\iffullversion
\newcommand\appendixref{Appendix}
\else
\newcommand\appendixref{full version of this paper~\cite{freegapinfo}}
\fi

\newif\ifshowdiff
\showdifffalse
\ifshowdiff
\newcommand{\hl}[1]{{\color{blue} #1}}
\else
\newcommand{\hl}[1]{{\color{defaultcolor} #1}}
\fi

\usepackage{graphicx}
\usepackage{balance}  

\usepackage[linesnumbered,lined,ruled,vlined]{algorithm2e}
\usepackage{comment}
\usepackage{amsmath, amsfonts, amssymb}
\usepackage{mathtools}
\usepackage{tikz}
\usetikzlibrary{arrows}
\usepackage{pgfplots}
\pgfplotsset{compat=newest}
\usepackage{enumitem}
\usepackage{makecell}
\usepackage{url}
\usepackage{todonotes}
\usepackage{subcaption}
\usepackage{hyperref}
\usepackage{float}
\usepackage{xcolor}
\usepackage{xspace}
\newcommand{\abs}[1]{\lvert#1\rvert}%
\newcommand{\norm}[1]{\lVert#1\rVert}%
\newcommand{\onenorm}[1]{\lVert#1\rVert_1}%
\newcommand{\set}[1]{\{#1\}}

\newcommand{\ZZ}{\mathbb{Z}}

\newcommand{\RR}{\mathbb{R}}

\newcommand{\PP}{\mathbb{P}}
\newcommand{\cI}{\mathcal{I}}
\newcommand{\cJ}{\mathcal{J}}

\newcommand{\io}{\mathcal{I}_\omega}
\newcommand{\ioc}{\mathcal{I}_\omega^c}
\newcommand{\jo}{\mathcal{J}_\omega}

\DeclareMathOperator{\lap}{Lap}
\DeclareMathOperator{\var}{Var}

\DeclareMathOperator{\len}{len}

\DeclareMathOperator{\cost}{cost}
\DeclareMathOperator{\otag}{tag}
\newcommand{\noisymax}{Noisy Max\xspace}
\newcommand{\gapmax}{Noisy-Max-with-Gap\xspace}
\newcommand{\topk}{Noisy Top-K\xspace}
\newcommand{\gaptopk}{Noisy-Top-K-with-Gap\xspace}
\newcommand{\topkmeasure}{Noisy-Top-K-with-Gap with Measures\xspace}
\newcommand{\svt}{Sparse Vector\xspace}
\newcommand{\gapsvt}{Sparse-Vector-with-Gap\xspace}
\newcommand{\adaptivesvt}{Adaptive-Sparse-Vector-with-Gap\xspace}
\newcommand{\svtmeasure}{Sparse-Vector-with-Gap with Measures\xspace}

\newcommand{\funcgaptopk}{NoisyTopK\xspace}
\newcommand{\functopkmeasure}{NoisyTopKMeasures\xspace}

\newcommand{\funcadaptivesvt}{AdaptiveSparseVector\xspace}

\newcommand{\vect}[1]{\boldsymbol{#1}}
\newcommand{\vq}{\vect{q}}

\newcommand{\vg}{\vect{g}}

\newcommand{\valpha}{\vect{\alpha}}

\newcommand{\vbeta}{\vect{\beta}}
\newcommand{\vxi}{\vect{\xi}}
\newcommand{\veta}{\vect{\eta}}
\newcommand{\vtheta}{\vect{\theta}}

\newcommand{\sdetemplate}[3]{S_{#1 #2:#3}}
\DeclareMathOperator{\sde}{\sdetemplate{}{D}{E}}
\DeclareMathOperator{\sdep}{\sdetemplate{}{D^\prime}{E}}
\DeclareMathOperator{\sdei}{\sdetemplate{}{D}{E_i}}
\DeclareMathOperator{\sdeip}{\sdetemplate{}{D^\prime}{E_i}}
\DeclareMathOperator{\sdo}{\sdetemplate{}{D}{\omega}}
\DeclareMathOperator{\sdop}{\sdetemplate{}{D^\prime}{\omega}}
\DeclareMathOperator{\smde}{\sdetemplate{M,}{D}{E}}
\DeclareMathOperator{\smdep}{\sdetemplate{M,}{D^\prime}{E}}
\DeclareMathOperator{\smdo}{\sdetemplate{M,}{D}{\omega}}

\newcommand{\altemplate}[1]{\ensuremath{\phi_{D,D^\prime\xspace{#1}}}}
\DeclareMathOperator{\alios}{\altemplate{,\omega^*}}
\DeclareMathOperator{\alio}{\altemplate{,\omega}}
\DeclareMathOperator{\ali}{\altemplate{}}

\newcommand{\alpsi}[1]{\ensuremath{\psi^{(#1)}_{D,D^\prime,\omega}}}

\usepackage{amsthm}
\newtheoremstyle{mystyle}
  {}
  {}
  {\itshape}
  {}
  {\scshape}
  {.}
  { }
  {}
\theoremstyle{mystyle}

\newtheorem{theorem}{Theorem}
\newtheorem{lemma}{Lemma}
\newtheorem{definition}{Definition}
\newtheorem{corollary}{Corollary}

\newdef{example}{Example}

\newcommand{\papertitle}{Free Gap Information from the Differentially Private Sparse Vector and Noisy Max Mechanisms}

\vldbTitle{\papertitle}
\vldbAuthors{Zeyu Ding, Yuxin Wang, Danfeng Zhang and Daniel Kifer}
\vldbDOI{https://doi.org/10.14778/xxxxxxx.xxxxxxx}
\vldbVolume{12}
\vldbNumber{xxx}
\vldbYear{2019}

\allowdisplaybreaks
\AtBeginDocument{\colorlet{defaultcolor}{.}}
\begin{document}


\title{\papertitle}



%
%
%
%

\numberofauthors{1} 

\author{
%
%
\alignauthor
Zeyu Ding, Yuxin Wang, Danfeng Zhang, Daniel Kifer \\\vspace{0.5em}
       \affaddr{Pennsylvania State University, University Park, PA 16802}\\
       \email{\{zyding, yxwang\}@psu.edu, \{zhang, dkifer\}@cse.psu.edu}
}
\date{1 May 2019}

\maketitle

\begin{abstract}
Noisy Max and Sparse Vector are selection algorithms for differential privacy and serve as building blocks for more complex algorithms. In this paper we show that both algorithms can release additional information for free (i.e., at no additional privacy cost). Noisy Max is used to return the approximate maximizer among a set of queries. We show that it can also release for free the noisy gap between the approximate maximizer and runner-up. \hl{This free information can improve the accuracy of certain subsequent counting queries by up to 50\%.} Sparse Vector is used to return a set of queries that are approximately larger than a fixed threshold. We show that it can adaptively control its privacy budget (use less budget for queries that are likely to be much larger than the threshold) \hl{in order to increase the amount of queries it can process}. \hl{These results follow from a careful privacy analysis.}
\end{abstract}

\section{Introduction}\label{sec:intro}
Industry and government agencies are increasingly adopting differential privacy \cite{dwork06Calibrating} to protect the confidentiality of users who provide data. Current and planned major applications include data gathering by Google \cite{rappor,prochlo}, Apple \cite{applediffp}, and Microsoft  \cite{DingKY17}; database querying by Uber \cite{elasticsensitivity}; and publication of population statistics at the U.S. Census Bureau \cite{ashwin08:map,onthemap,Haney:2017:UCF,abowd18kdd}.

The accuracy of differentially private data releases is very important in these applications. One way to improve accuracy is to increase the value of the privacy parameter $\epsilon$, known as the privacy loss budget, as it provides a tradeoff between an algorithm's utility and its privacy protections. However, values of $\epsilon$ that are deemed too high can subject a company to criticisms of not providing enough privacy \cite{TangPLApple}. For this reason, researchers invest significant effort in tuning algorithms \cite{diffperm,ectelo,pythia,deepdp,pate2,fanaeepour2018histogramming} and privacy analyses \cite{BS2016:zcdp,M2017:Renyi,pate2,ErlingssonFMRTT19} to provide better utility while using smaller privacy budgets.

Differentially private algorithms are built out of smaller components called \emph{mechanisms} \cite{pinq}. Popular mechanisms include the Laplace Mechanism \cite{dwork06Calibrating}, Geometric Mechanism \cite{universallyUtilityMaximizingPrivacyMechanisms},  \noisymax \cite{diffpbook} and \svt \cite{diffpbook,lyu2017understanding}. As we will explain in this paper, the latter two mechanisms, \noisymax and \svt, inadvertently throw away information that is useful for designing accurate algorithms. Our contribution is to present novel variants of these mechanisms that provide more functionality at the same privacy cost (under pure differential privacy).

Given a set of queries, \noisymax returns the identity (not value) of the query that is likely to have the largest value -- it adds noise to each query answer and returns the index of the query with the largest noisy value. Meanwhile, \svt takes a stream of queries and a predefined public threshold $T$. It tries to return the identities (not values) of the first $k$ queries that are likely larger than the threshold. To do so, it adds noise to the threshold. Then, as it sequentially processes each query, it outputs ``$\top$'' or ``$\bot$'', depending on whether the noisy value of the current query is larger or smaller than the noisy threshold. The mechanism terminates after $k$ ``$\top$'' outputs.

In recent work \cite{shadowdp}, using program verification tools, Wang et al. showed that \svt can provide additional information \emph{at no additional cost to privacy}. That is, when \svt returns ``$\top$'' for a query, it can also return the gap between its noisy value and the noisy threshold.\footnote{This was a surprising result given the number of incorrect attempts at improving \svt based on flawed manual proofs \cite{lyu2017understanding} and shows the power of automated program verification techniques.} We refer to their algorithm as \gapsvt.

Inspired by this program verification work, we propose novel variations of \svt and {\noisymax} \hl{that add new functionality}. For \svt, we show that in addition to releasing this gap information, even stronger improvements are possible -- we present an adaptive version that can answer more queries than before by controlling how much privacy budget it uses to answer each query. The intuition is that we would like to spend less of our privacy budget for queries that are probably much larger than the threshold (compared to queries that are probably closer to the threshold). A careful accounting of the privacy impact shows that this is possible. \hl{Our experiments} confirm that \adaptivesvt can answer many more queries than the prior versions \cite{lyu2017understanding,diffpbook,shadowdp} at the same privacy cost. 

For \noisymax, we show that it too inadvertently throws away information. Specifically, \emph{at no additional cost to privacy}, it can release an estimate of the gap between the largest and second largest queries (we call the resulting mechanism \gapmax). We  generalize this result to \topk~-- showing that one can release an estimate of the \emph{identities} of the $k$ largest queries and, at no extra privacy cost, release noisy estimates of the pairwise \emph{gaps} (differences) among the top $k+1$ queries. 

The extra \hl{noisy gap} information opens up new directions in the construction of differentially private algorithms \hl{and can be used to improve accuracy of certain subsequent queries}. For instance, one common task is to select the approximate top $k$ queries and then use additional privacy loss budget to \hl{obtain noisy answers to these queries.} We \hl{show that a postprocessing step can combine these noisy answers with gap information to improve accuracy by up to 50\% for counting queries}. 
\hl{We provide similar applications for the free gap information in \svt.}

We prove our results using the alignment of random variables technique \cite{lyu2017understanding,diffperm,shadowdp,lightdp}, which is based on the following question: if we change the input to a program, how must we change its random variables so that output remains the same? \hl{This technique is used to prove the correctness of almost all pure differential privacy mechanisms \cite{diffpbook} but needs to be used in sophisticated ways to prove the correctness of the more advanced algorithms \cite{lyu2017understanding,diffperm,diffpbook,shadowdp,lightdp}}. \hl{Nevertheless, alignment of random variables is} often used incorrectly (as discussed by Lyu et al. \cite{lyu2017understanding}). Thus a secondary contribution of our work is to lay out the precise steps and conditions that must be checked and to provide helpful lemmas that ensure these conditions are met. 


To summarize, our contributions are as follows:
\begin{itemize}[leftmargin=0.5cm,itemsep=0cm,topsep=0.5em,parsep=0.5em]
    \item We provide a simplified template for writing correctness proofs for intricate differentially private algorithms.
    \item Using this technique, we propose and prove the correctness of \hl{two new mechanisms: \gaptopk and \adaptivesvt.}
    These algorithms improve on the original versions of {\noisymax} and {\svt} \hl{by taking advantage of \emph{free} information (i.e., information that can be released at no additional privacy cost) that those algorithms inadvertently throw away.} 
    \item We demonstrate some of the uses of the gap information that is provided by these new mechanisms.  When an algorithm needs to use \noisymax or \svt to select some queries and then measure them (i.e., obtain their noisy answers), we show how the gap information from our new mechanisms can be used to improve the accuracy of the noisy measurements. We also show how the gap information in \svt can be used to estimate the confidence that a query's true answer really is larger than the threshold. 
    \item We empirically evaluate these new mechanisms on a variety of datasets to demonstrate their improved utility.
\end{itemize}

In Section \ref{sec:related}, we discuss related work. We present background and notation in Section \ref{sec:background}. We present simplified proof templates for randomness alignment in Section \ref{sec:raf}. We present \hl{\gaptopk} in Section \ref{sec:noisygap} and \hl{\adaptivesvt} in Section \ref{sec:improvesparse}. We present  experiments in Section \ref{sec:experiments}, \hl{proofs underlying the alignment of randomness framework} in Section \ref{subsec:incproofs} and conclusions in Section \ref{sec:conc}. The rest of our proofs can be found in the \appendixref.

\section{Related Work}\label{sec:related}
Selection algorithms, such as Exponential Mechanism \cite{exponentialMechanism,RaskhodnikovaS16}, \svt \cite{diffpbook,lyu2017understanding}, and \noisymax \cite{diffpbook} are used to select a set of items (typically queries) from a much larger set. They have applications in hyperparameter tuning \cite{diffperm,LiuPSPC}, iterative construction of microdata \cite{mwem}, feature selection \cite{GuhaSmith13}, frequent itemset mining \cite{BhaskarDFP}, exploring a privacy/accuracy tradeoff \cite{LigettNRWW17}, data pre-processing \cite{ChenDPRD}, etc.

Various generalizations have been proposed \cite{LigettNRWW17,BeimelNS16,GuhaSmith13,RaskhodnikovaS16,ChaudhuriLMM,LiuPSPC}. Liu and Talwar \cite{LiuPSPC} and Raskhodnikova and Smith \cite{RaskhodnikovaS16} extend the exponential mechanism for arbitrary sensitivity queries.
Beimel et al. \cite{BeimelNS16} and Thakurta and Smith \cite{GuhaSmith13} use the propose-test-release framework \cite{diffrobust} to find a gap between the best and second best queries and, if the gap is large enough, release the identity of the best query. These two algorithms rely on a relaxation of differential privacy called approximate $(\epsilon,\delta)$-differential privacy \cite{dworkKMM06:ourdata} and can fail to return an answer (in which case they return $\perp$). Our algorithms work with pure $\epsilon$-differential privacy. Chaudhuri et al. \cite{ChaudhuriLMM} also proposed a large margin mechanism (with approximate differential privacy) which finds a large gap separating top queries from the rest and returns one of them.

There have also been unsuccessful attempts to generalize selection algorithms such as \svt (incorrect versions are catalogued by Lyu et al. \cite{lyu2017understanding}), which has sparked innovations in program verification for differential privacy (e.g.,  \cite{Barthe16,Aws:synthesis,lightdp,shadowdp}) \hl{with techniques such as probabilistic coupling \cite{Barthe16} and a simplification based on randomness alignment \cite{lightdp}. These are similar to ideas behind handwritten proofs \cite{diffperm,diffpbook,lyu2017understanding} -- they consider what changes need to be made to random variables in order to make two executions of a program, with different inputs, produce the same output. It is a powerful technique that is behind almost all proofs of differential privacy, but is very easy to apply incorrectly \cite{lyu2017understanding}. In this paper, we state and prove a more general version of this technique in order to prove correctness of our algorithms and also provide additional results that simplify the application of this technique.}


\section{Notation and Background}\label{sec:background}
 In this paper, we use the following notation. $D$ and $D^\prime$ refer to databases. We use the notation $D\sim D^\prime$ to represent adjacent databases.\footnote{The notion of adjacency depends on the application. Some papers define it as $D$ can be obtained from $D^\prime$ by modifying one record \cite{dwork06Calibrating} or by adding/deleting one record \cite{Dwork06diffpriv}.} $M$ denotes a randomized algorithm whose input is a database.  $\Omega$ denotes the range of $M$ and $\omega\in \Omega$ denotes a specific output of $M$. We use $E\subseteq \Omega$ to denote a set of possible outputs. Because $M$ is randomized, it also relies on a random noise vector $H\in \RR^\infty$ (which usually consists of independent zero-mean Laplace random variables). This noise sequence is infinite, but of course $M$ will only use a finite-length prefix of $H$. When we need to draw attention to the noise, we use the notation $M(D,H)$ to indicate the execution of $M$ with database $D$ and randomness coming from $H$. Otherwise we use the notation $M(D)$. Define $\smde = \set{H\mid M(D,H)\in E }$ to be the set
 of noise vectors that allow $M$, on input $D$, to produce an output in the set $E\subseteq \Omega$. To avoid overburdening the notation, we write $\sde$ for $\smde$ and $\sdep$ for $\smdep$ when $M$ is clear from the context. When $E$ consists of a single point $\omega$, we write these sets as $\sdo$ and $\sdop$. 
 This notation is summarized in the table below.

\begin{table}[!ht]
\begin{center}
\caption{Notation}\label{tab:notation}
\begin{tabular}{c c}
\Xhline{1.5\arrayrulewidth}
\textbf{Symbol} & \textbf{Meaning} \\ \hline
$M$ &  \hl{randomized algorithm}  \\ 
$D, D'$& database \\ 
$D\sim D^\prime$ & $D$ is adjacent to $D^\prime$\\
$H=(\eta_1,\eta_2,\ldots)$ & input noise vector \\
$\Omega$ & the space of all output of $M$  \\ 
$\omega$ & a possible output; $\omega\in\Omega$  \\ 
$E$ & a set of possible outputs; $E\subseteq\Omega$  \\
$\sde = \smde$ & $\set{H \mid M(D,H)\in E}$ \\ 
$\sdo=\smdo$ & $\set{H \mid M(D,H)= \omega}$ \\ 
\Xhline{1.5\arrayrulewidth}
\end{tabular}
\end{center}
\end{table}

\subsection{Formal Privacy} Differential privacy \cite{dwork06Calibrating,Dwork06diffpriv,diffpbook} is currently the gold standard for releasing privacy-preserving information about a database. It has a parameter $\epsilon>0$ known as the privacy loss budget. The smaller it is, the more privacy is provided. Differential privacy bounds the effect of one record on the output of the algorithm (for small $\epsilon$, the probability of any output is barely affected by any person's record).
\begin{definition}[Pure Differential Privacy \cite{Dwork06diffpriv}]
Given an $\epsilon>0$, a randomized algorithm $M$ with output space $\Omega$ satisfies (pure) $\epsilon$-differential privacy if for all $E\subseteq\Omega$ and all pairs of adjacent databases $D\sim D^\prime$, the following holds:
\begin{equation}\label{eq:dpdef}
    \PP(M(D,H)\in E) \leq e^\epsilon \PP(M(D',H)\in E)
\end{equation}
where the probability is only over the randomness of $H$. 
\end{definition}
With the notation in Table~\ref{tab:notation}, the differential privacy condition from Equation \eqref{eq:dpdef} is $\PP(\sde)\leq e^\epsilon \PP(\sdep)$.

Differential privacy enjoys the following nice properties:
\begin{itemize}[leftmargin=0.5cm,itemsep=0cm,topsep=0.5em,parsep=0.5em]
    \item Resilience to Post-Processing. If we apply an algorithm $A$ to the output of an $\epsilon$-differentially private algorithm $M$, then the composite algorithm $A\circ M$ still satisfies $\epsilon$-differential privacy. In other words, privacy is not reduced by post-processing.
    
    \item Composition. If $M_1,M_2,\dots, M_k$ satisfy differential privacy with privacy loss budgets $\epsilon_1,\dots,\epsilon_k$, the algorithm that runs all of them and releases their outputs satisfies $(\sum_i\epsilon_i)$-differential privacy.
\end{itemize}

Many differentially private algorithms take advantage of the Laplace mechanism \cite{exponentialMechanism}, which provides a noisy answer to a vector-valued query $\vq$ based on its $L_1$ \emph{global sensitivity} $\Delta_{\vq}$, defined as follows:
\begin{definition}[$L_1$ Global Sensitivity \cite{diffpbook}] The global sensitivity of a query $\vq$ is
$\Delta_{\vq} = \sup_{D\sim D'} \onenorm{\vq(D)-\vq(D')}$.
\end{definition}
\begin{theorem}[Laplace Mechanism \cite{dwork06Calibrating}]
Given a privacy loss budget $\epsilon$, consider the mechanism that returns $\vq(D)+ H$, where $H$ is a vector of independent random samples from the Laplace$(\Delta_{\vq}/\epsilon)$ distribution with mean 0 and scale parameter $\Delta_{\vq}/\epsilon$. This Laplace mechanism satisfies $\epsilon$-differential privacy. 
\end{theorem}
\hl{Other kinds of additive noise distributions that can be used in place of Laplace include Discrete Laplace \cite{universallyUtilityMaximizingPrivacyMechanisms} (when all query answers are integers or multiples of a common base) and Staircase \cite{staircase}.}
\section{Randomness Alignment}\label{sec:raf}
To establish that the algorithms we propose are differentially private, we use an idea called \emph{randomness alignment} that previously had been used to prove the privacy of a variety of sophisticated algorithms \cite{diffpbook,lyu2017understanding,diffperm} and incorporated into verification/synthesis tools \cite{lightdp,shadowdp,Aws:synthesis}. While powerful, this technique is also easy to use incorrectly \cite{lyu2017understanding}, as there are many technical conditions that need to be checked. 
In this section, we present results (namely Lemma \ref{lem:alignmentbound}) that significantly simplify this process and make it easy to prove the correctness of our proposed algorithms. 

In general, to prove $\epsilon$-differential privacy for an algorithm $M$, one needs to show $P(M(D)\in E)\leq e^\epsilon P(M(D^\prime)\in E)$ for all pairs of adjacent databases $D\sim D^\prime$ and sets of possible outputs $E\subseteq \Omega$. In our notation, this inequality is represented as $\PP(\sde)\leq e^\epsilon \PP(\sdep)$. Establishing such inequalities is often done with the help of a function $\ali$, called a \emph{randomness alignment} (there is a function $\ali$ for every pair $D\sim D^\prime$), that maps noise vectors $H$ into noise vectors $H^\prime$ so that $M(D^\prime,H^\prime)$ produces the same output as  $M(D, H)$. Formally, 
\begin{definition}[Randomness Alignment] Let $M$ be \hl{a} randomized algorithm. Let $D\sim D^\prime$ be a pair of adjacent databases. A \emph{randomness alignment} is a function $\ali: \RR^\infty \rightarrow \RR^\infty$ such that  for all $H$ on which $M(D,H)$ terminates, we have $M(D,H) = M(D',\ali(H))$.
\end{definition}

\begin{example}\label{ex:lp}
Let $D$ be a database that records the salary of every person, which is guaranteed to be between 0 and 100. Let $q(D)$ be the sum of the salaries in $D$. The sensitivity of $q$ is thus $100$. Let $H=(\eta_1,\eta_2,\dots)$ be a vector of independent Laplace$(100/\epsilon)$ random variables. The Laplace mechanism outputs $q(D) + \eta_1$ (and ignores the remaining variables in $H$). For every pair of adjacent databases $D\sim D^\prime$, one can define the corresponding randomness alignment $\ali(H)=H^\prime=(\eta^\prime_1,\eta^\prime_2,\dots)$, where $\eta^\prime_1=\eta_1 + q(D)-q(D^\prime)$ and $\eta^\prime_i=\eta_i$ for $i>1$. Note that $q(D)+\eta_1 = q(D^\prime)+\eta^\prime_1$, so the output of $M$ remains the same.
\end{example}

 In practice, $\ali$ is constructed locally (piece by piece) as follows. For each possible output $\omega\in \Omega$, one defines a function $\alio$ that maps noise vectors $H$ into noise vectors $H^\prime$ with the following properties:  if $M(D,H)=\omega$ then $M(D^\prime, H^\prime)=\omega$ (that is, $\alio$ only cares about what it takes to produce the specific output $\omega$). We obtain our randomness alignment $\ali$ in the obvious way by piecing together the $\alio$ as follows:  $\ali(H)=\alios(H)$, where $\omega^*$ is the output of $M(D,H)$. Formally,

\begin{definition}[Local Alignment] Let $M$ be \hl{a} randomized algorithm. Let $D\sim D^\prime$ be a pair of adjacent databases and $\omega$ a possible output of $M$. A \emph{local alignment} for $M$ is a function $\alio: \sdo \rightarrow \sdop$ (see notation in Table \ref{tab:notation}) such that  for all $ H\in \sdo$, we have $M(D,H) = M(D',\alio(H))$.
\end{definition}

\begin{example}\label{ex:thresh}
Continuing the setup from Example \ref{ex:lp}, consider the mechanism $M_1$ that, on input $D$, outputs $\top$ if $q(D)+\eta_1 \geq 10,000$ (i.e. if the noisy total salary is at least $10,000$) and $\perp$ if $q(D)+\eta_1 < 10,000$. Let $D^\prime$ be a database that differs from $D$ in the presence/absence of one record. Consider the local alignments $\altemplate{,\top}$ and $\altemplate{,\perp}$ defined as follows.  $\altemplate{,\top}(H)=H^\prime=(\eta^\prime_1,\eta^\prime_2,\dots)$ where $\eta^\prime_1 = \eta_1 + 100$ and $\eta^\prime_i=\eta_i$ for $i> 1$; and $\altemplate{,\perp}(H)=H^{\prime\prime}=(\eta^{\prime\prime}_1,\eta^{\prime\prime}_2,\dots)$ where  $\eta^{\prime\prime}_1 = \eta_1 - 100$ and $\eta^{\prime\prime}_i=\eta_i$ for $i> 1$. Clearly, if $M_1(D,H)=\top$ then $M_1(D^\prime, H^\prime)=\top$ and if $M_1(D, H)=\perp$ then $M_1(D^\prime, H^{\prime\prime})=\perp$. We piece these two local alignments together to create a randomness alignment $\ali(H)=H^*=(\eta^*_1,\eta^*_2,\dots)$ where:
\begin{align*}
    \eta^*_1 &= \begin{cases}
        \eta_1 + \hl{100} & \text{ if } M(D,H) =\top \\
        &\text{ (i.e. $q(D)+\eta_1 \geq 10,000$)}\\
        \eta_1 - \hl{100} & \text{ if } M(D,H) = \perp\\
        &\text{ (i.e. $q(D)+\eta_1 < 10,000$)}\\
    \end{cases}\\
    \eta^*_i &= \eta_i \text{ for }  i>1
\end{align*}
\end{example}

%

\paragraph*{Special properties of alignments}
Not all alignments can be used to prove differential privacy. In this section we discuss some additional properties that help prove differential privacy. 
We first make two mild assumptions about the mechanism $M$: (1) it terminates with probability\footnote{That is, for each input $D$,  there might be some random vectors $H$ for which $M$ does not terminate, but the total probability of these vectors is 0, so we can ignore them.} one  and (2) based on the output of $M$, we can determine how many random variables it used. The vast majority of differentially private algorithms in the literature satisfy these properties.

We next define two properties of  a local alignment: whether it is \emph{acyclic} and what its \emph{cost} is.
\begin{definition}[Acyclic]\label{def:acyclic}
Let $M$ be a randomized algorithm. Let $\alio$ be a local alignment for $M$. For any $H=(\eta_1, \eta_2,\dots)$, let $H^\prime=(\eta_1^\prime, \eta_2^\prime, \dots)$ denote $\alio(H)$. We say that $\alio$ is acyclic if there exists a permutation $\pi$ and piecewise differentiable functions $\alpsi{j}$ such that:
\begin{align*}
    \eta^\prime_{\pi(1)} &= \eta_{\pi(1)} + \text{number that only depends on $D$, $D^\prime$, $\omega$}\\
    \eta^\prime_{\pi(j)} &= \eta_{\pi(j)} + \alpsi{j}(\eta_{\pi(1)},\dots,\eta_{\pi(j-1)}) ~\text{for $j\geq 2$}
\end{align*}
\end{definition}
Essentially, a local alignment $\alio$ is acyclic if there is some ordering of the variables so that $\eta^\prime_j$ is the sum of $\eta_j$ and a function of the variables that came earlier in the ordering. The local alignments  $\altemplate{,\top}$ and  $\altemplate{,\perp}$ from Example \ref{ex:thresh} are both acyclic  (in general, each local alignment function is allowed to have its own specific ordering and differentiable functions $\alpsi{j}$). The pieced-together randomness alignment $\ali$ itself need not be acyclic. 

\begin{definition}[Alignment Cost]\label{def:alignmentcost}
Let $M$ be a randomized algorithm that uses $H$ as its source of randomness. Let $\alio$ be a local alignment for $M$.  
 For any $H=(\eta_1, \eta_2,\dots)$, let $H^\prime=(\eta_1^\prime, \eta_2^\prime, \dots)$ denote $\alio(H)$. \hl{Suppose each $\eta_i$ is generated independently from a
distribution $f_i$ with the property that $\log(f_i(x)/f_i(y)) \leq |x-y|/\alpha_i$ for all $x,y$ in the domain of $f_i$  -- this includes the Laplace$(\alpha_i)$ distribution along with Discrete Laplace \cite{universallyUtilityMaximizingPrivacyMechanisms} and Staircase \cite{staircase}}. Then the cost of $\alio$ is defined as:
\begin{align*}
\text{cost}(\alio) &= \sum_i |\eta_i - \eta^\prime_i|/\alpha_i    
\end{align*}
\end{definition}

The following lemma uses those properties to establish that $M$ satisfies $\epsilon$-differential privacy.

\begin{lemma}\label{lem:alignmentbound}  
Let $M$ be \hl{a randomized}  algorithm with input randomness $H=(\eta_1,\eta_2,\dots)$. If the following conditions are satisfied, then $M$ satisfies $\epsilon$-differential privacy.
\begin{enumerate}[leftmargin=5mm,parsep=0.25em,itemsep=0.25em,topsep=0.5em,label=(\roman{enumi})]
    \item $M$ terminates with probability 1.
    \item The number of random variables used by $M$ can be determined from its output.
    \item\label{conditioniii} Each $\eta_i$ is generated independently from a distribution  \hl{$f_i$ with the property that $\log(f_i(x)/f_i(y)) \leq |x-y|/\alpha_i$ for all $x,y$ in the domain of $f_i$ (such as Laplace$(\alpha_i)$).}
    \item\label{conditioniv} \underline{For every $D\sim D^\prime$ and $\omega$} there exists a local alignment $\alio$ that is acyclic with $\cost(\alio)\leq \epsilon$.
    \item \underline{For each $D\sim D^\prime$} the number of distinct local alignments is countable. That is, the set $\{\alio\mid \omega\in\Omega\}$ is countable (i.e., for many choices of $\omega$ we get the same exact alignment function).
\end{enumerate}
\end{lemma}
We defer the proof to Section \ref{subsec:incproofs}.

\begin{example}\label{ex:raf}
Consider the randomness alignment $\ali$ from Example~\ref{ex:lp}. We can define all of the local alignments $\alio$ to be the same function: $\alio(H)=\ali(H)$. Clearly $\cost(\alio) = \sum_{i=0}^\infty \frac{\epsilon}{100}\abs{\eta'_i-\eta_i}\ =\frac{\epsilon}{100}\abs{q(D')-q(D)} \leq 
\epsilon$. For Example~\ref{ex:thresh}, there are two acyclic local alignments $\altemplate{\top}$ and $\altemplate{\bot}$, both have $\cost = 100\cdot \frac{ \epsilon}{100} = \epsilon$. The other conditions in Lemma~\ref{lem:alignmentbound} are trivial to check. Thus both mechanisms satisfy $\epsilon$-differential privacy by Lemma~\ref{lem:alignmentbound}.
\end{example}

\section{Improving Noisy Max}\label{sec:noisygap}
In this section, we present novel variations of the \noisymax mechanism \cite{diffpbook}. Given a list of queries with sensitivity 1, the purpose of \noisymax is to estimate the identity (i.e., index) of the largest query. We show that, in addition to releasing this index, it is possible to release a numerical estimate of the gap between the values of the largest and second largest queries. This extra information comes at no additional cost to privacy, meaning that the original \noisymax mechanism threw away useful information. \hl{This result can be generalized to the setting in which one wants to estimate the identities of the top $k$ queries -  we can release (for free) all of the gaps between each top $k$ query and the next best query (i.e., the gap between the best and second best queries, the gap between the second and third best queries, etc). When a user  subsequently asks for a noisy answer to each of the returned queries, we show how the gap information can be used to reduce squared error by up to 50\% (for counting queries).}


\subsection{\gaptopk}
\hl{
Our proposed \gaptopk mechanism is shown in Algorithm \ref{alg:gaptopk} (the function $\arg\max_c$ returns the top c items). We can obtain the classical Noisy Max algorithm \cite{diffpbook} from it by setting $k=1$ and throwing away the gap information (the boxed items on Lines \ref{line:gapdef} and \ref{line:gapreturn}). The \gaptopk algorithm takes as input a sequence of $n$ queries $q_1,\dots, q_n$, each having sensitivity 1. It adds Laplace noise to each query. It returns the indexes $j_1,\dots, j_k$ of the $k$ queries with the largest noisy values in descending order. Furthermore, for each of these top $k$ queries $q_{j_i}$, it releases the noisy gap between the value of $q_{j_i}$ and the value of the next best query. Our key contribution in this section is the observation that these gaps can be released for free. That is, the classical Top-K algorithm, which does not release the gaps, satisfies $\epsilon$-differential privacy. But, our improved version has exactly the same privacy cost yet is strictly better because of the extra information it can release.
}

\begin{algorithm}[ht]
\SetKwProg{Fn}{function}{\string:}{}
\SetKwFunction{Test}{\funcgaptopk}
\SetKwInOut{Input}{input}
\DontPrintSemicolon
\Input{$\vq$: a list of $n$ queries of global sensitivity 1\\
$D$: database, $k$: \# of indexes, $\epsilon$: privacy budget}
\Fn{\Test{$\vq$, $D$, $k$, $\epsilon$}}{
\ForEach{$\mathtt{i} \in \set{1, \cdots, n}$}{
$\eta_i \gets \lap(2k/\epsilon)$\;
$\widetilde{q}_i \gets q_i(D)+\eta_i$\;
}
$(j_1,\ldots, j_{k+1}) \gets \arg\max_{k+1}(\widetilde{q}_1,\ldots, \widetilde{q}_n)$\;
\ForEach{$\mathtt{i} \in \set{1, \cdots, k }$}{
\fbox{$g_{i} \gets \widetilde{q}_{j_i} - \widetilde{q}_{j_{i+1}}$ ~~~~~~\tcp{$i^{th}$ gap}}\label{line:gapdef}
}
\Return $((j_1\fbox{, $g_1$}),\ldots, (j_k \fbox{, $g_k$}))$\label{line:gapreturn} \tcp{$k$ queries}
}
\caption{\gaptopk}
\label{alg:gaptopk}
\end{algorithm}

We \hl{emphasize} that keeping the noisy gaps hidden does not decrease the privacy cost. Furthermore, this algorithm gives estimates of the pairwise gaps between any pair of the $k$ queries it selects. For example, suppose we are interested in estimating the gap between the $a^{\text{th}}$ largest and $b^\text{th}$ largest queries (where $a < b\leq k$). This is equal to $\sum_{i=a}^{b-1} g_i$ because:
$
    \sum_{i=a}^{b-1} g_i = \sum_{i=a}^{b-1} (\widetilde{q}_{j_i} - \widetilde{q}_{j_{i+1}}) = \widetilde{q}_{j_a} - \widetilde{q}_{j_b}
$
and hence its variance is $\var(\widetilde{q}_{j_a} - \widetilde{q}_{j_b}) = 16k^2/\epsilon^2$.

The original Noisy Top-K mechanism satisfies $\epsilon$-differential privacy. In the special case that all the $q_i$ are counting queries\footnote{i.e., when a person is added to a database, the value of each query either stays the same or increases by 1.} then it satisfies $\epsilon/2$-differential privacy \cite{diffpbook}. We will show the same properties for \gaptopk. 
We prove the privacy property in this section and then in Section \ref{sec:blue} we show how to use this gap information. However, first it is important to discuss the difference between the theoretical analysis of Noisy Top-K \cite{diffpbook} and its practical implementation on finite-precision computers.

\paragraph*{Implementation issues} The analysis of the original \noisymax mechanism assumed the use of \emph{true} Laplace noise (a continuous distribution) so that ties are impossible between the largest and second largest noisy queries \cite{diffpbook}. On finite precision computers, ties are possible (breaking the privacy proof \cite{diffpbook}) \hl{with some probability $\delta$}. \hl{Thus, one would settle for a slightly weaker guarantee called  $(\epsilon,\delta)$-differential  privacy \cite{dworkKMM06:ourdata}. It roughly states that the privacy conditions of pure $\epsilon$-differential privacy fail with probability at most $\delta$}. 
\hl{In practice, one would discretize the Staircase \cite{staircase} or Laplace distribution so that it outputs a multiple of some base $\gamma$.}
%
%
%
In the \appendixref, we show that if there are $n$ queries with sensitivity 1 and discretized Lapalce($1/\epsilon)$ noise with base $\gamma$ is added to each of them, the probability of a tie is upper bounded by $\delta=\epsilon \gamma n^2$ \hl{(similar calculations can be performed with the Staircase distribution)}. Thus this is an upper bound on the probability that the differential privacy guarantees will fail. Typically, one would expect $\gamma$ to be close to machine epsilon (e.g., $\approx 2^{-52}$) so the probability of a tie is negligible. In this section we will also analyze our algorithms under the assumption of continuous noise. Thus the privacy guarantees can fail with this negligible probability $\delta$ (hence satisfying $(\epsilon,\delta)$-differential privacy \cite{dworkKMM06:ourdata}).

\hl{
\paragraph*{Local alignment}
To prove the privacy of Algorithm \ref{alg:gaptopk}, we need to create a local alignment function for each possible pair $D\sim D^\prime$ and output $\omega$. Note that our mechanism uses precisely $n$ random variables. Let $H=(\eta_1,\eta_2,\dots)$ where $\eta_i$ is the noise that should be added to the $i^\text{th}$ query. We view the output $\omega=((j_1,g_1),\ldots, (j_k, g_k))$ as $k$ pairs where in the $i^{\text{th}}$ pair $(j_i,g_i)$, the first component $j_i$ is the index of $i^\text{th}$ largest noisy query and the second component $g_i$ is the gap in noisy value between the $i^\text{th}$ and $(i+1)^\text{th}$ largest noisy queries. As discussed in the implementation issues, we will base our analysis on continuous noise so that the probability  of ties among the top $k+1$ noisy queries is $0$. Thus each gap is positive: $g_i>0$. }

\hl{
Let $\io = \set{j_1,\ldots, j_k}$ and $\ioc=\set{1, \ldots, n}\setminus\io$.  I.e., $\io$ is the index set of the $k$ largest noisy queries selected by the algorithm and $\ioc$ is the index set of all unselected queries. 
For $H\in\sdo$ define $\alio(H) = H'=(\eta'_1, \eta'_2, \ldots)$ as
\begin{equation}\label{eq:gaptopk_align}
\eta_i' =\begin{cases}
\eta_i &i\in\ioc\\
\eta_i + q_i-q'_i+\max\limits_{l\in\ioc}(q_l'+\eta_l)- \max\limits_{l\in\ioc}(q_l+\eta_l)  &i\in\io
\end{cases}    
\end{equation}
The idea behind this local alignment is simple: we want to keep the noise of the losing queries the same (when the input is $D$ or its neighbor $D^\prime)$. But, for each of the $k$ selected queries, we want to align its noise to make sure it wins by the same amount when the input is $D$ or its neighbor $D'$. 
}

\hl{
\begin{lemma}\label{lem:topk_align}
Let $M$ be the \gaptopk algorithm.
For all $D\sim D^\prime$ and $\omega$, the functions $\alio$ defined above are acyclic local alignments for $M$. 
Furthermore, for every pair $D\sim D^\prime$, there are countably many distinct $\alio$.
\end{lemma}
}
\hl{
\begin{proof}
Given $D\sim D^\prime$ and $\omega=((j_1,g_1),\ldots, (j_k,g_k))$,
for any $H=(\eta_1,\eta_2,\dots)$ such that $M(D,H)=\omega$, let $H^\prime=(\eta_1^\prime, \eta_2^\prime, \dots) = \alio(H)$.
We show that $M(D', H') = \omega$. Since $\alio$ is identity on components $i\in \ioc$, we have  $\max\limits_{l\in\ioc}(q'_l+\eta'_l) = \max\limits_{l\in\ioc}(q'_l+\eta_l)$.
So, for the $k^\text{th}$ selected query:
\begin{align*}
(q'_{j_k} +\eta'_{j_k}) - \max\limits_{l\in\ioc}(q'_l+\eta_l')
&=(q'_{j_k} +\eta'_{j_k}) - \max\limits_{l\in\ioc}(q'_l+\eta_l)\\
&\hspace{-1cm}= (q_{j_k} + \eta_{j_k}) - \max\limits_{l\in\ioc}(q_l+\eta_l) = g_{k} > 0
\end{align*}
where the last line follows from Equation \ref{eq:gaptopk_align}.
This means on $D'$ the noisy query with index $j_k$ is larger than the best of the unselected noisy queries by the same margin as it is on $D$. Furthermore, for all $1\leq i < k$, we have
\begin{align*}
&(q'_{j_i} +\eta'_{j_i}) - (q'_{j_{i+1}}+\eta'_{j_{i+1}})\\ 
=&(q_{j_i} +\eta_{j_i} + \max\limits_{l\in\ioc}(q_l'+\eta_l)- \max\limits_{l\in\ioc}(q_l+\eta_l)) \\
&\hspace{12mm}- (q_{j_{i+1}}+\eta_{j_{i+1}} + \max\limits_{l\in\ioc}(q_l'+\eta_l)- \max\limits_{l\in\ioc}(q_l+\eta_l))\\
=&(q_{j_i} +\eta_{j_i}) - (q_{j_{i+1}}+\eta_{j_{i+1}}) = g_{i}>0.
\end{align*}
In other words, the query with index $j_i$ is still the $i^\text{th}$ largest query on $D'$ by the same margin.
Thus $M(D', H') = \omega$.

The local alignments are clearly acyclic (any permutation that puts $\ioc$ before $\io$ does the trick). Also, note that $\alio$ only depends on $\omega$ through $\io $ (the indexes of the $k$ largest queries). There are $n$ queries and therefore $\binom{n}{k} = \frac{n!}{(n-k)!k!}$ distinct $\alio$.
\end{proof}
}
\paragraph*{Alignment cost and privacy}
To establish the alignment cost, we need the following lemma and definition.
\hl{
\begin{lemma}\label{lem:maxsen}
Let $(x_1,\ldots, x_m), (x'_1,\ldots, x'_m)\in \RR^m$ be such that $\forall  i, \abs{x_i-x'_i}\leq 1$. Then $\abs{\max_i \set{x_i} - \max_i \set{x'_i}}\leq 1$.
\end{lemma}
\begin{proof}
Let $s$ be an index that maximizes $x_i$ and let $t$ be an index that maximizes $x'_i$. Without loss of generality, assume $x_{s} \geq x'_{t}$. Then
$x_s \geq x'_t \geq x'_s \geq x_s-1$. Hence
$ \abs{x_s - x'_t} = x_s - x'_t \leq x_s - (x_s-1) = 1.$
\end{proof}
}
\hl{
\begin{definition}[Monotonicity]\label{def:mono}
A list of numerical queries $\vq=(q_1, q_2,\ldots)$ is monotonic if for all pair of adjacent databases $D\sim D'$ we have either $\forall i: q_i(D) \leq q_i(D')$, or $\forall i: q_i(D) \geq q_i(D')$.
\end{definition}
Counting queries are clearly monotonic. Now we can establish the privacy property of our algorithm.}
\hl{
\begin{theorem}\label{thm:gaptopk_cost}
The \gaptopk mechanism satisfies $\epsilon$-differential privacy. If all of the queries are counting queries, then it satisfies $\epsilon/2$-differential privacy.
\end{theorem}}
\hl{
\begin{proof} First we bound the cost of the alignment function defined in \eqref{eq:gaptopk_align}. Recall that
the $\eta_i$'s are independent $\lap(2k/\epsilon)$ random variables. By Definition \ref{def:alignmentcost}
\begin{align*}
\cost&(\alio) =\sum_{i=1}^\infty\abs{\eta_i'-\eta_i} \frac{\epsilon}{2k}\\
=&\frac{\epsilon}{2k}\sum_{i\in\cI_\omega}\abs{q_i-q_i'+ \max\limits_{l\in\ioc}(q_l'+\eta_l)-  \max\limits_{l\in\ioc}(q_l+\eta_l)}.
\end{align*}
By the global sensitivity assumption we have $\abs{q_i-q_i'}\leq 1$.
Apply Lemma \ref{lem:maxsen} to the vectors $(q_l+\eta_l)_{l\in\ioc}$ and $(q'_l+\eta_l)_{l\in\ioc}$, we have $\abs{ \max\limits_{l\in\ioc}(q_l'+\eta_l)-  \max\limits_{l\in\ioc}(q_l+\eta_l)}\leq 1$. Therefore,
\begin{align*}
  &\abs{q_i-q_i'+ \max\limits_{l\in\ioc}(q_l'+\eta_l)-  \max\limits_{l\in\ioc}(q_l+\eta_l)} \\
  \leq& \abs{q_i-q_i'} + \abs{ \max\limits_{l\in\ioc}(q_l'+\eta_l)-  \max\limits_{l\in\ioc}(q_l+\eta_l)} \leq 1+1 =2.  
\end{align*}
Furthermore, if $\vq$ is monotonic, then 
\begin{itemize}[leftmargin = 5mm]
    \item either $\forall i: q_i \leq q'_i$ in which case $q_{i}-q'_{i}\in [-1,0]$ and $\max\limits_{l\in\cI_\omega^c}(q_l'+\eta_l)-  \max\limits_{l\in\cI_\omega^c}(q_l+\eta_l)\in [0,1]$,
    \item or $\forall i: q_i \geq q'_i$ in which case $q_{i}-q'_{i}\in [0,1]$ and $\max\limits_{l\in\cI_\omega^c}(q_l'+\eta_l)-  \max\limits_{l\in\cI_\omega^c}(q_l+\eta_l)\in [-1,0]$.
\end{itemize}
In both cases we have $q_{i}-q_{i}'+ \max\limits_{l\in\cI_\omega^c}(q_l'+\eta_l)-  \max\limits_{l\in\cI_\omega^c}(q_l+\eta_l) \in [-1,1]$ so $|q_{i}-q_{i}'+ \max\limits_{l\in\cI_\omega^c}(q_l'+\eta_l)-  \max\limits_{l\in\cI_\omega^c}(q_l+\eta_l)|\leq 1$.
Therefore,
\begin{align*}
\cost(\alio)
=&\frac{\epsilon}{2k}\sum_{i\in\cI_\omega}\abs{q_i-q_i'+ \max\limits_{l\in\ioc}(q_l'+\eta_l)-  \max\limits_{l\in\ioc}(q_l+\eta_l)}\\
\leq& \frac{\epsilon}{2k}\sum_{i\in\cI_\omega} 2 \quad (\textrm{or } \frac{\epsilon}{2k}\sum_{i\in\cI_\omega} 1 \textrm{ if } \vq \textrm{ is monotonic})\\
=& \frac{\epsilon}{2k}\cdot 2 \abs{\io} \quad (\textrm{or } \frac{\epsilon}{2k}\cdot \abs{\io} \textrm{ if } \vq \textrm{ is monotonic}) \\
=& \epsilon \quad (\textrm{or } {\epsilon}/{2} \textrm{ if } \vq \textrm{ is monotonic}).
\end{align*}
Conditions (i) through (iii) of Lemma \ref{lem:alignmentbound} are trivial to check, (iv) and (v) follow from Lemma \ref{lem:topk_align} and the above bound on cost. Therefore, Theorem \ref{thm:gaptopk_cost} follows from Lemma \ref{lem:alignmentbound}.
\end{proof}}





\subsection{Utilizing Gap Information}\label{sec:blue}
Let us consider one scenario that takes advantage of the gap information. Suppose a data analyst is interested in the identities and values of the top $k$ queries. A typical approach would be to split the privacy budget $\epsilon$ in half -- use $\epsilon/2$ of the budget to identify the top $k$ queries using \gaptopk. The remaining $\epsilon/2$ budget is evenly divided between the selected queries and is used to  obtain noisy measurements  (i.e. add Laplace$(2k/\epsilon)$ noise to each query answer). These measurements will have variance $\sigma^2=8k^2/\epsilon^2$. In this section we show how to use the gap information from \gaptopk and postprocessing to improve the accuracy of these measurements.

\hl{
\paragraph*{Problem statement}
Let $q_1,\dots, q_k$ be the true answers of the top $k$ queries that are selected by Algorithm \ref{alg:gaptopk}. Let $\alpha_1,\ldots, \alpha_k$ be their noisy measurements. Let $g_1,\ldots,g_{k-1}$ be the noisy gaps between $q_1,\ldots,q_k$ that are obtained from Algorithm \ref{alg:gaptopk} for free.
Then $\alpha_i=q_i + \xi_i$ where each $\xi_i$ is a Laplace$(2k/\epsilon)$ random variable and $g_i =  q_i  + \eta_i - q_{i+1} - \eta_{i+1}$ where each $\eta_i$ is a Laplace$(4k/\epsilon)$ random variable, or a Laplace$(2k/\epsilon)$ random variable if the query list is monotonic (recall the mechanism was run with a privacy budget of $\epsilon/2$).  Our goal is then to find the \emph{best linear unbiased estimate} (BLUE) \cite{lehmann1998} $\beta_i$ of $q_i$ in terms of the measurements $\alpha_i$ and gap information $g_i$.}

%

\hl{
\begin{theorem}\label{thm:blue} With notations as above let $\vq=[q_1, \ldots, q_k]^T$, $\valpha=[\alpha_1, \ldots, \alpha_k]^T$ and $\vg=[g_1, \ldots, g_{k-1}]^T$. Suppose the ratio $\var(\xi_i) : \var(\eta_i)$ is equal to $1:\lambda$.
Then the BLUE of $\vq$ is $\vbeta = \frac{1}{(1+\lambda)k} (X\valpha + Y\vg)$ where
\[
X = \begin{bmatrix}
1+\lambda k & 1 & \cdots & 1 \\
1 & 1+\lambda k & \cdots & 1 \\
\vdots & \vdots & \ddots & \vdots \\
1 & 1 & \cdots & 1+\lambda k \\
\end{bmatrix}_{k\times k}
\]
\[
Y = \left(\begin{bmatrix}
k-1 & k-2 & \cdots &1 \\
k-1 & k-2 & \cdots &1 \\
k-1 & k-2 & \cdots &1 \\
\vdots & \vdots & \ddots &\vdots\\
k-1 & k-2 & \cdots & 1 \\
\end{bmatrix} - 
\begin{bmatrix}
0 & 0 & \cdots &0 \\
k & 0 & \cdots &0 \\
k & k & \cdots &0 \\
\vdots & \vdots & \ddots &0 \\
k & k & \cdots &k \\
\end{bmatrix}
\right)_{k\times (k-1)}
\]
\end{theorem}}
For proof, see the \appendixref.
Even though this is a matrix multiplication, it is easy to see that it translates into the following algorithm that is linear in $k$:
\begin{enumerate}[leftmargin=0.5cm,itemsep=0cm,topsep=0.5em,parsep=0.5em]
    \item Compute $\alpha = \sum_{i=1}^k \alpha_i$ and $p = \sum_{i=1}^{k-1} (k-i)g_i$.
    \item Set $p_0=0$. For $i=1, \ldots, k-1$ compute the prefix sum $p_i= \sum_{j=1}^i g_j = p_{i-1} + g_i$.
    \item For $i=1, \ldots, k$, set $\beta_i = 
    (\alpha + \lambda k\alpha_i + p - kp_{i-1})/(1+\lambda)k$.
\end{enumerate}

Now, each $\beta_i$ is an estimate of the value of $q_i$. How does it compare to the direct measurement $\alpha_i$ (which has variance $\sigma^2 = 8k^2/\epsilon^2$)? The following result compares the expected error of $\beta_i$ (which used the direct measurements and the gap information) with the expected error of using only the direct measurements (i.e., $\alpha_i$ only).

\hl{
\begin{corollary}\label{cor:blue_var}
For all $i=1,\ldots, k$, we have \[\frac{E(\abs{\beta_i-q_i}^2)}{E(\abs{\alpha_i-q_i}^2)} 
= \frac{1+\lambda k}{k+\lambda k } = \frac{\var(\xi_i) + k\var(\eta_i)}{k(\var(\xi_i)+\var(\eta_i))}.\]
\end{corollary}
For proof, see the \appendixref. In the case of counting queries, we have $\var(\xi_i) = \var(\eta_i) = 8k^2/\epsilon^2$ and thus $\lambda = 1$. The error reduction rate is $\frac{k-1}{2k}$ which is close to $50\%$ when $k$ is large. Our experiments in Section~\ref{sec:experiments} confirm this theoretical result.}

\section{Improving Sparse Vector}\label{sec:improvesparse}
In this section we propose a novel variant that can answer more queries than both the original \svt \cite{diffpbook,lyu2017understanding} and the \gapsvt of Wang et al. \cite{shadowdp}. We also discuss how the free gap information can be used.

\subsection{\adaptivesvt}\label{sec:sparsegap}
The Sparse Vector techniques are designed to solve the following problem in a privacy-preserving way: given a stream of queries (with sensitivity 1), find the first $k$ queries whose answers are larger than a public threshold $T$. This is done by adding noise to the queries and threshold and finding the first $k$ queries whose noisy answers exceed the noisy threshold. Sometimes this procedure creates a feeling of regret -- if these $k$ queries are much larger than the threshold, we could have used more noise (hence consumed less privacy budget) to achieve the same result.
In this section, we show that Sparse Vector can be made adaptive -- so that it will probably use more noise (less privacy budget) for the larger queries. This means if the first $k$ queries are very large, it will still have privacy budget left over to find additional queries that are likely to be over the threshold. Our Adaptive Sparse Vector is shown in Algorithm \ref{alg:adaptivesvt}. 


\begin{algorithm}[ht]
\SetKwProg{Fn}{function}{\string:}{}
\SetKwFunction{Test}{\funcadaptivesvt}
\SetKwInOut{Input}{input}
\SetKwInOut{Output}{output}
\DontPrintSemicolon
\Input{$\vq$: a list of queries of global sensitivity 1\\
$D$: database, $\epsilon$: privacy budget, $T$: threshold\\
\hl{$k$: minimum number of above-threshold\\ \hspace{9pt} queries algorithm is able to output}}
\Fn{\Test{$\vq$, $D$, $T$, $k$, $\epsilon$}}{
\hl{$\epsilon_0 \gets\theta \epsilon$}; \hl{$\epsilon_1 \gets (1-\theta)\epsilon /k$}; 
\hl{$\epsilon_2 \gets \epsilon_1 / 2$}; $\sigma \gets \hl{4\sqrt{2}/\epsilon_2}$\;\label{line:theta}
$\eta \gets \lap(1/\epsilon_0)$\;
$\widetilde{T} \gets T + \eta$\; 
$\mathtt{cost} \gets \epsilon_0$\;
\ForEach{$\mathtt{i} \in \set{1, \cdots, \len(\vq)}$}{
\hl{$\xi_i \gets \lap(2/\epsilon_2)$; $\eta_i \gets \lap(2/\epsilon_1)$}\label{line:adaptivesvt_noise}\;
\uIf{$q_i(D) + \xi_i - \widetilde{T} \geq \sigma$ $\label{line:adaptivesvt_top_start}$}
{
\textbf{output:} ($\top$, $q_i(D) + \xi_i - \widetilde{T}$, $\mathtt{bud\_used}=\hl{\epsilon_2}$)\;
$\mathtt{cost} \gets \mathtt{cost} + \hl{\epsilon_2}$ $\label{line:adaptivesvt_top_stop}$\; 
}
\uElseIf{$q_i(D) + \eta_i - \tilde{T} \geq 0$ $\label{line:adaptivesvt_middle_start}$}{
\textbf{output:} ($\top$, $q_i(D) + \eta_i - \widetilde{T}$, $\mathtt{bud\_used}=\hl{\epsilon_1}$)\;
$\mathtt{cost} \gets \mathtt{cost} + \hl{\epsilon_1}$ $\label{line:adaptivesvt_middle_stop}$\;
}
\Else{
    \textbf{output:} ($\bot$, $\mathtt{bud\_used=0}$)\;
    }
    \lIf{$\mathtt{cost} > \epsilon - \hl{\epsilon_1}$}{\textbf{break}}
}
}
\caption{\adaptivesvt. \hl{The hyperparameter $\theta\in (0,1)$  controls the budget allocation between threshold and queries.}}
\label{alg:adaptivesvt}
\end{algorithm}

The main idea behind this algorithm is that, given a target privacy budget $\epsilon$ and an integer $k$, the algorithm will create three \hl{budget} parameters: \hl{$\epsilon_0$ (budget for the threshold), $\epsilon_1$ (baseline budget for each query) and $\epsilon_2$ (smaller alternative budget for each query, $\epsilon_2<\epsilon_1$). The privacy budget allocation between threshold and queries is controlled by a hyperparameter $\theta\in (0,1)$ on Line \ref{line:theta}. These budget parameters} are used as follows. First, the algorithm adds Laplace$(1/\epsilon_0)$ noise to the threshold and consumes $\epsilon_0$ of the privacy budget. Then, when a query comes in, the algorithm first adds a lot of noise (i.e., \hl{Laplace$(2/\epsilon_2)$}) to the query. The first ``if'' branch checks if this value is much larger than the noisy threshold (i.e. checks if the gap is $\geq \sigma$ for some\footnote{In our algorithm, we set $\sigma$ to be 2 standard deviations of the Laplace($2/\epsilon_2$) distribution.} $\sigma$). If so, then it outputs the following three items: (1) $\top$, (2) the noisy gap, and (3) the amount of privacy budget \hl{used for this query} (which is \hl{$\epsilon_2$}). The use of alignments will show that failing this ``if'' branch consumes no privacy budget. If the first ``if'' branch fails, then the algorithm adds more moderate noise to the query answer (i.e., \hl{Laplace$(2/\epsilon_1)$}). If this noisy value is larger than the noisy threshold, the algorithm outputs: (1$^\prime$) $\top$, (2$^\prime$) the noisy gap, and (3$^\prime$) the amount of privacy budget consumed (i.e., \hl{$\epsilon_1$}). If this ``if'' condition also fails, then the algorithm outputs: (1$^{\prime\prime}$) $\perp$ and (2$^{\prime\prime}$) the privacy budget consumed ($0$ in this case).

To summarize, for each query, if the top branch succeeds then the privacy budget consumed is \hl{$\epsilon_2$}, if the middle branch succeeds, the privacy cost is \hl{$\epsilon_1$}, and if the bottom branch succeeds, there is no additional privacy cost. These properties can be easily seen by focusing on the local alignment -- if $M(D,H)$ produces a certain output, how much does $H$ need to change to get a noise vector $H^\prime$ so that $M(D^\prime,H^\prime)$ returns the same exact output.

\paragraph*{Local alignment}
To create a local alignment for each pair $D\sim D^\prime$, let $H=(\eta,\xi_1,\eta_1,\xi_2,\eta_2, \ldots)$ where $\eta$ is the noise added to the threshold $T$, and $\xi_i$ (resp. $\eta_i$) is the noise that should be added to the $i^\text{th}$ query $q_i$ in Line~\ref{line:adaptivesvt_top_start} (resp. Line~\ref{line:adaptivesvt_middle_start}), if execution ever reaches that point. We view the output $\omega=(w_1,\dots,w_s)$ as a variable-length sequence where each $w_i$ is either $\perp$ or a nonnegative gap (we omit the $\top$ as it is redundant), together with a $\otag\in\set{0,\epsilon_1,\epsilon_2}$ indicating  which branch $w_i$ is from (and the privacy budget consumed to output $w_i$). 
Let $\io = \set{i\mid \otag(w_i) = \epsilon_2}$ and $\jo = \set{i\mid \otag(w_i) = \epsilon_1}$. That is, $\io$ is the set of indexes where the output is a gap from the top branch, and $\jo$ is the set of indexes where the output is a gap from the middle branch.
For $H\in\sdo$ define $\alio(H)=H'=(\eta',\xi_1',\eta_1',\xi_2', \eta_2', \ldots)$ where
\begin{equation}\label{eq:adaptivesvt_align}
    \begin{aligned}
    \eta' &= \eta + 1, \\
    (\xi_i', ~~~\eta_i') &=\begin{cases}
(\xi_i + 1 + q_i - q_i',~~~\eta_i ), &i \in\cI_\omega\\
(\xi_i,~~~ \eta_i + 1 + q_i - q_i'), &i \in\cJ_\omega\\
(\xi_i, ~~~\eta_i), & \textrm{otherwise}
\end{cases}
    \end{aligned}
\end{equation}
In other words, we add 1 to the noise that was added to the threshold (thus if the noisy $q(D)$ failed a specific branch, the noisy $q(D^\prime)$ will continue to fail it because of the higher noisy threshold). If a noisy $q(D)$ succeeded in a specific branch, we adjust the query's noise so that the noisy version of $q(D^\prime)$ will succeed in that same branch.

\hl{
\begin{lemma}\label{lem:adaptivesvt_align}
Let $M$ be the \adaptivesvt algorithm.
For all $D\sim D^\prime$ and $\omega$, the functions $\alio$ defined above are acyclic local alignments for $M$. 
Furthermore, for every pair $D\sim D^\prime$, there are countably many distinct $\alio$.
\end{lemma}
}

\begin{proof}
\hl{
Pick an adjacent pair $D\sim D^\prime$ and an $\omega=(w_1,\dots,w_s)$.
For a given $H=(\eta, \xi_1,\eta_1,\dots)$ such that $M(D,H)=\omega$, let $H^\prime=(\eta^\prime,\xi_1^\prime, \eta_1^\prime, \dots) = \alio(H)$.
Suppose $M(D',H') = \omega' =(w_1', \ldots, w_{t}')$. Our goal is to show $\omega^\prime=\omega$. Choose an $i\leq \min(s,t)$. 
\begin{itemize}[leftmargin=0.5cm,itemsep=0cm,topsep=0.5em,parsep=0.5em]
    \item If $i\in\cI_\omega$, then by \eqref{eq:adaptivesvt_align} we have
\begin{align*}
    q_i' &+ \xi_i' - (T + \eta') = q_i' + \xi_i + 1 + q_i - q_i' - (T + \eta + 1) \\
    & = q_i + \xi_i -(T+\eta)  \geq \sigma.
\end{align*}
This means the first ``if'' branch succeeds in both executions and the gaps are the same . Therefore, $w_i' = w_i$.

\item If $i\in\cJ_\omega$, then by \eqref{eq:adaptivesvt_align} we have
\begin{align*}
    q_i' &+ \xi_i' - (T + \eta') = q_i' + \xi_i  - (T + \eta + 1) \\
    & = q_i' - 1 + \xi_i -(T+\eta) \leq q_i + \xi_i - (T+\eta) < \sigma,\\
    q_i' &+ \eta_i' - (T + \eta') = q_i' +\eta_i + 1 + q_i - q_i'  - (T + \eta + 1) \\
    &=  q_i + \eta_i - (T+\eta)  \geq 0.
\end{align*}
The first inequality is due to the sensitivity restriction: $\abs{q_i-q_i'}\leq 1 \implies q_i'-1\leq q_i$. These two equations mean that the first ``if'' branch fails and the second ``if'' branch succeeds in both executions, and the gaps are the same.
Hence $w_i' =  w_i$.
\item If $i\not\in \cI_\omega\cup\cJ_\omega$, then by a similar argument we have 
\begin{align*}
   q_i' + \xi_i' - (T + \eta') &\leq q_i + \xi_i - (T+\eta) < \sigma,\\
    q_i' + \eta_i' - (T + \eta') &\leq q_i + \eta_i - (T+\eta) < 0.
\end{align*}
Hence both executions go to the last ``else'' branch and $w_i' = (\bot, 0) = w_i$.
\end{itemize}
Therefore for all $1\leq i \leq \min(s,t)$, we have $w_i'=w_i$. That is, either $\omega'$ is a prefix of $\omega$, or vice versa. Let $\vq$ be the vector of queries passed to the algorithm and let $\len(\vq)$ be the number of queries it contains (which can be finite or infinity). By the termination condition of Algorithm \ref{alg:adaptivesvt} we have two possibilities.  
\begin{itemize}[leftmargin=0.5cm,itemsep=0cm,topsep=0.5em,parsep=0.5em]
\item $s=\len(\vq)$: in this case there is still enough privacy budget left after answering $s-1$ above-threshold queries, and we must have $t=\len(\vq)$ too because $M(D',H')$ will also run through all the queries (it cannot stop until it has exhausted the privacy budget or hits the end of the query sequence). 
\item $s<\len(\vq)$: in this case the privacy budget is exhausted after outputting $w_s$ and we must also have $t=s$.  
\end{itemize}
Thus $t=s$ and hence $\omega' = \omega$. 
The local alignments are clearly acyclic (e.g., use the identity permutation). 
Note that $\alio$ only depends on $\omega$ through $\io$ and $\jo$ (the sets of queries whose noisy values were larger than the noisy threshold). There are only countably many possibilities for $\io$ and $\jo$ and thus countably many distinct $\alio$.}
\end{proof}

\hl{
\paragraph*{Alignment cost and privacy}
Now we establish the alignment cost and the privacy property of Algorithm \ref{alg:adaptivesvt}.
\begin{theorem}\label{thm:adaptivesvt_cost}
The \adaptivesvt satisfies $\epsilon$-differential privacy.
\end{theorem}
\begin{proof}
Again, the only thing nontrivial is to bound the alignment cost. We use the $\epsilon_0, \epsilon_1, \epsilon_2$ and $\epsilon$ defined in Algorithm \ref{alg:adaptivesvt}. 
From \eqref{eq:adaptivesvt_align} we have 
\begin{align*}
\cost(\alio) &=\epsilon_0\abs{\eta'-\eta} + \sum_{i=1}^\infty \left(\frac{\epsilon_2}{2}\abs{\xi_i'-\xi_i} + \frac{\epsilon_1}{2}\abs{\eta_i'-\eta_i}\right)\\
&= \epsilon_0 + \sum_{i\in \cI_\omega} \frac{\epsilon_2}{2}\abs{1 + q_i - q_i'} + \sum_{i\in \cJ_\omega} \frac{\epsilon_1}{2}\abs{1 + q_i - q_i'} \\
&\leq  \epsilon_0 + \epsilon_2\abs{\cI_\omega} + \epsilon_1\abs{\cJ_\omega} \leq \epsilon.
\end{align*}
The first inequality is from sensitivity assumption: $\abs{1 + q_i - q_i'} \leq 1 + \abs{q_i - q_i'} \leq 2$. The second inequality is from loop invariant on Line 16: $\epsilon_0 + \epsilon_2\abs{\cI_\omega} + \epsilon_1\abs{\cJ_\omega} =\mathtt{cost} \leq \epsilon - \epsilon_1 + \max(\epsilon_1, \epsilon_2) = \epsilon$.
\end{proof}
}

\hl{
We note that if we remove the first branch of Algorithm \ref{alg:adaptivesvt} (Line \ref{line:adaptivesvt_top_start} through  \ref{line:adaptivesvt_top_stop}) or set $\sigma = \infty$,  we recover the \gapsvt algorithm of Wang. et al. \cite{shadowdp}.
Also, Algorithm \ref{alg:adaptivesvt} can be easily extended with multiple additional ``if'' branches. For simplicity we do not include such variations. In our setting, $\epsilon_2=\epsilon_1/2$  so, theoretically, if queries are very far from the threshold, our adaptive version of Sparse Vector will be able to find twice as many of them as the non-adaptive version. Lastly, if all queries are monotonic queries, then Algorithm \ref{alg:adaptivesvt} can be further improved: we can use $\lap(1/\epsilon_2)$ and $\lap(1/\epsilon_1)$ noises instead in Line \ref{line:adaptivesvt_noise}.\footnote{\hl{In the case of monotonic queries, if $\forall i: q_i \geq q^\prime_i$, then the alignment changes slightly: we set $\eta^\prime=\eta$ (the random variable added to the threshold) and set the adjustment to noise in winning ``if'' branches to $q_i-q^\prime_i$ instead of $1+q_i-q^\prime_i$  (hence cost terms become $|q_i-q^\prime_i|$ instead of $|1+q_i-q^\prime_i|$). If $\forall i: q_i \leq q^\prime_i$ then we keep the original alignment but in the cost calculation we note that $|1+q_i-q^\prime_i|\leq 1$ (due to the monotonicity and sensitivity).}}
}

\subsection{Utilizing Gap Information}\label{sec:gapsvt_measures}
When \gapsvt or \adaptivesvt returns a gap $\gamma_i$ for a query $q_i$, we can add to it the public threshold $T$. This means $\gamma_i+T$ is an estimate of the value of $q_i(D)$. We can ask two questions: how can we improve the accuracy of this estimate and how can we be confident that the true answer $q_i(D)$ is really larger than the threshold $T$?

\paragraph*{Lower confidence interval}
Recall that the randomness in the gap in \adaptivesvt (Algorithm \ref{alg:adaptivesvt}) is of the form  $\eta_i - \eta$ where $\eta$ and $\eta_i$ are independent zero mean Laplace variables with scale $1/\epsilon_0$ and \hl{$1/\epsilon_*$, where $\epsilon_*$ is either $\epsilon_1$ or $\epsilon_2$, depending on the branch.} The random variable $\eta_i-\eta$ has the following lower tail bound:
\begin{lemma}\label{lem:confint}
For any $t\geq 0$ we have 
\[ \PP(\eta_i - \eta \geq -t ) = \begin{cases}
1 - \frac{\epsilon_0^2 e^{-\epsilon_\ast t} - \epsilon_\ast^2e^{-\epsilon_0t}}{2(\epsilon_0^2 - \epsilon_\ast^2)} & \epsilon_0\neq \epsilon_\ast \\
1 - (\frac{2+\epsilon_0t}{4})e^{-\epsilon_0t} & \epsilon_0 = \epsilon_\ast
\end{cases}
\]
\end{lemma}
For proof see the \appendixref.
For any confidence level, say 95\%, we can use this result to find a number $t_{.95}$ such that 
$\PP((\eta_i-\eta) \geq -t_{.95}) = .95$. This is a lower confidence bound, so that the true value $q_i(D)$ is $\geq$ our estimated value  $\gamma_i + T$ minus  $t_{.95}$ with probability $0.95$.

\paragraph*{Improving accuracy}
To improve accuracy, one can split the privacy budget $\epsilon$ in half. The first half $\epsilon'\equiv \epsilon/2$ can be used to run \gapsvt (or \adaptivesvt) and the second half $\epsilon''\equiv \epsilon/2$ can be used to provide an independent noisy measurement of the selected queries (i.e. if we selected $k$ queries, we add Laplace$(k/\epsilon'')$ noise to each one). Suppose the selected queries are $q_1,\dots, q_k$, the noisy gaps are $\gamma_1,\dots, \gamma_k$ and the independent noisy measurements  are $\alpha_1,\dots,\alpha_k$.
The noisy estimates can be combined together with the gaps to get improved estimates $\beta_i$ of $q_i(D)$ in the standard way (inverse-weighting by variance):
$$\beta_i = \left(\frac{\alpha_i}{\var(\alpha_i)} + \frac{\gamma_i+T}{\var(\gamma_i)}\right) \bigg/ \left(\frac{1}{\var(\alpha_i)} + \frac{1}{\var(\gamma_i)}\right).$$

As shown in ~\cite{lyu2017understanding}, the optimal budget allocation between threshold noise and query noises \emph{within} SVT (and therefore also \gapsvt) is the ratio $1:(2k)^\frac{2}{3}$. Under this setting, we have
$\var(\gamma_i) = 8(1+(2k)^\frac{2}{3})^3 /\epsilon^2$. Also, we know $\var(\alpha_i) = 8k^2/\epsilon^2$. Therefore,
$$
\frac{E( |\beta_i - q_i|^2)}{E( |\alpha_i - q_i|^2)} = 
\frac{ \var(\beta_i)}{ \var(\alpha_i)} = \frac{(1+ \sqrt[3]{4k^2})^3}{ (1+ \sqrt[3]{4k^2})^3 + k^2} < 1.
$$
Since $\lim\limits_{k\to\infty} \frac{(1+ \sqrt[3]{4k^2})^3}{ (1+ \sqrt[3]{4k^2})^3 + k^2} = \frac{4}{5}$, the improvement in accuracy approaches $20\%$ as $k$ increases.
\hl{For monotonic queries, the optimal budget allocation \emph{within} SVT is $1:k^{\frac{2}{3}}$. Then we have $\var(\gamma_i)=8(1+k^{\frac{2}{3}})^3/\epsilon^2$ and the error reduction rate is $1-\frac{(1+ \sqrt[3]{k^2})^3}{ (1+ \sqrt[3]{k^2})^3 + k^2}$ which is close to 50\% when $k$ is large. Our experiments in Section~\ref{sec:experiments} confirm this improvement.}

\section{Experiments}\label{sec:experiments}

We now evaluate the algorithms proposed in this paper.

\begin{figure*}[!ht]
\centering
\captionsetup{width=.95\linewidth}
\begin{subfigure}[t]{0.4\textwidth}
\captionsetup{width=1.1\linewidth,format=hang}
\includegraphics[width=\linewidth]{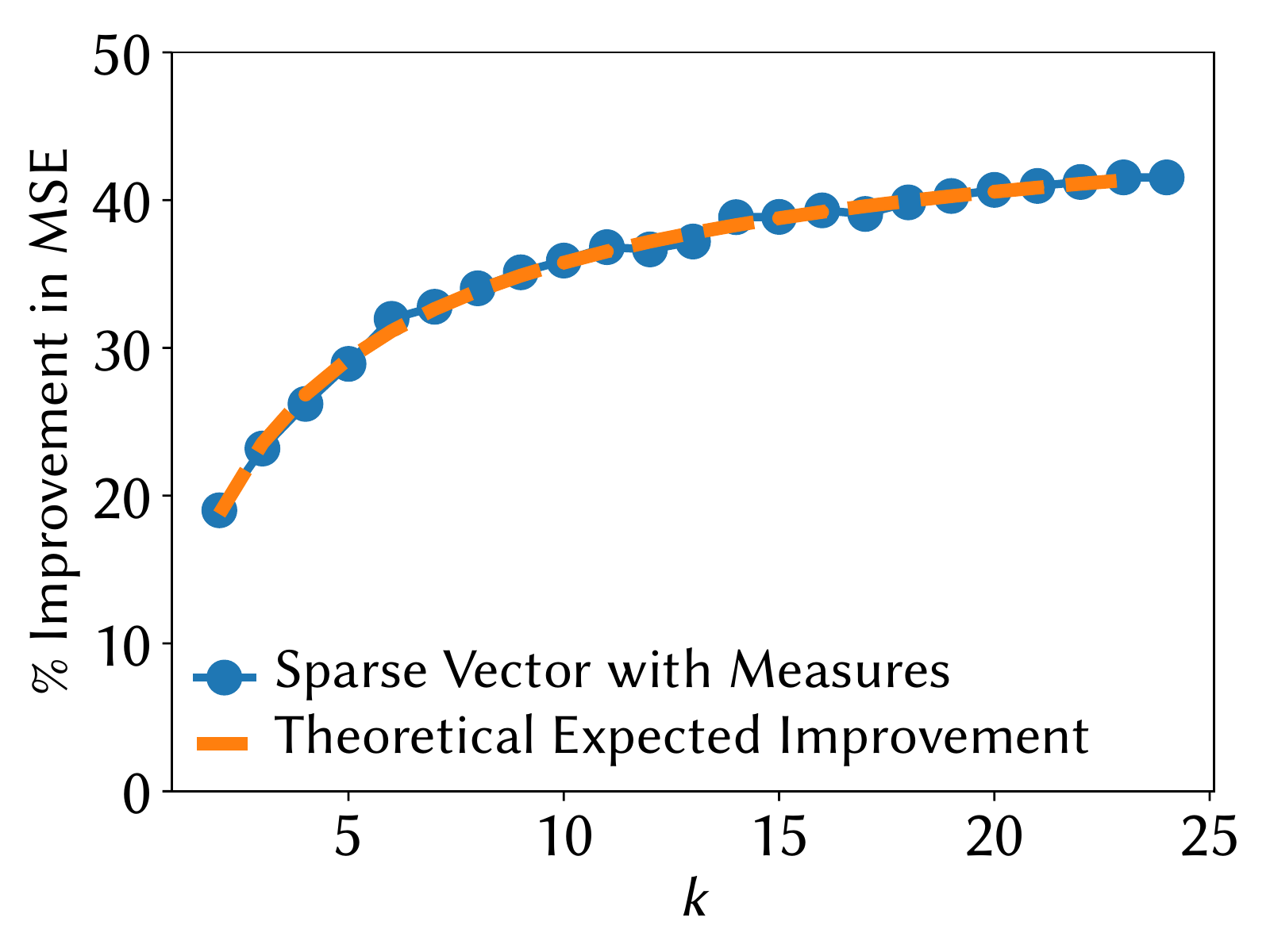} 
\caption{\svtmeasure, BMS-POS.\label{fig:svt_measure_bms_pos_counting}}
\end{subfigure}%
\hspace{0.1\textwidth}
\begin{subfigure}[t]{0.4\textwidth}
\captionsetup{width=\linewidth,format=hang}
\includegraphics[width=\linewidth]{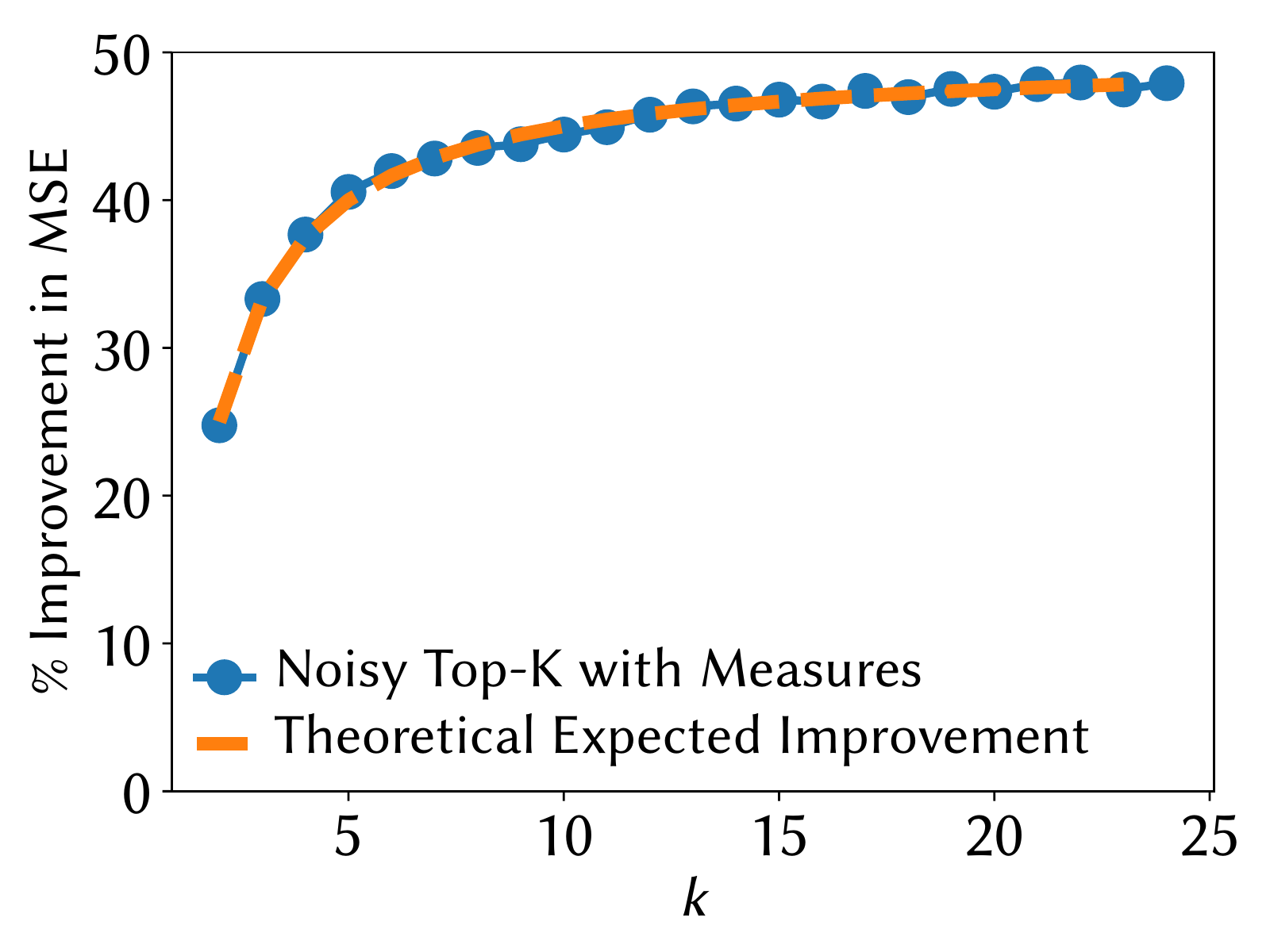} 
\caption{\topkmeasure, BMS-POS.\label{fig:noisy_topk_measure_bms_pos_counting}}
\end{subfigure}%
\caption{\hl{Improvement percentage of Mean Squared Error on monotonic queries, for different $k$, for Sparse Vector with Gap and Noisy Top-K with Gap when half the privacy budget is used for query selection and the other half is used for measurement of their answers. Privacy budget $\epsilon = 0.7$.\label{fig:noisy_topk_measure_counting} \label{fig:svt_measure_counting}}}
\end{figure*}

\begin{figure*}[!ht]
\centering
\captionsetup{width=.95\linewidth}
\begin{subfigure}[t]{0.4\textwidth}
\captionsetup{width=1.1\linewidth,format=hang}
\includegraphics[width=\linewidth]{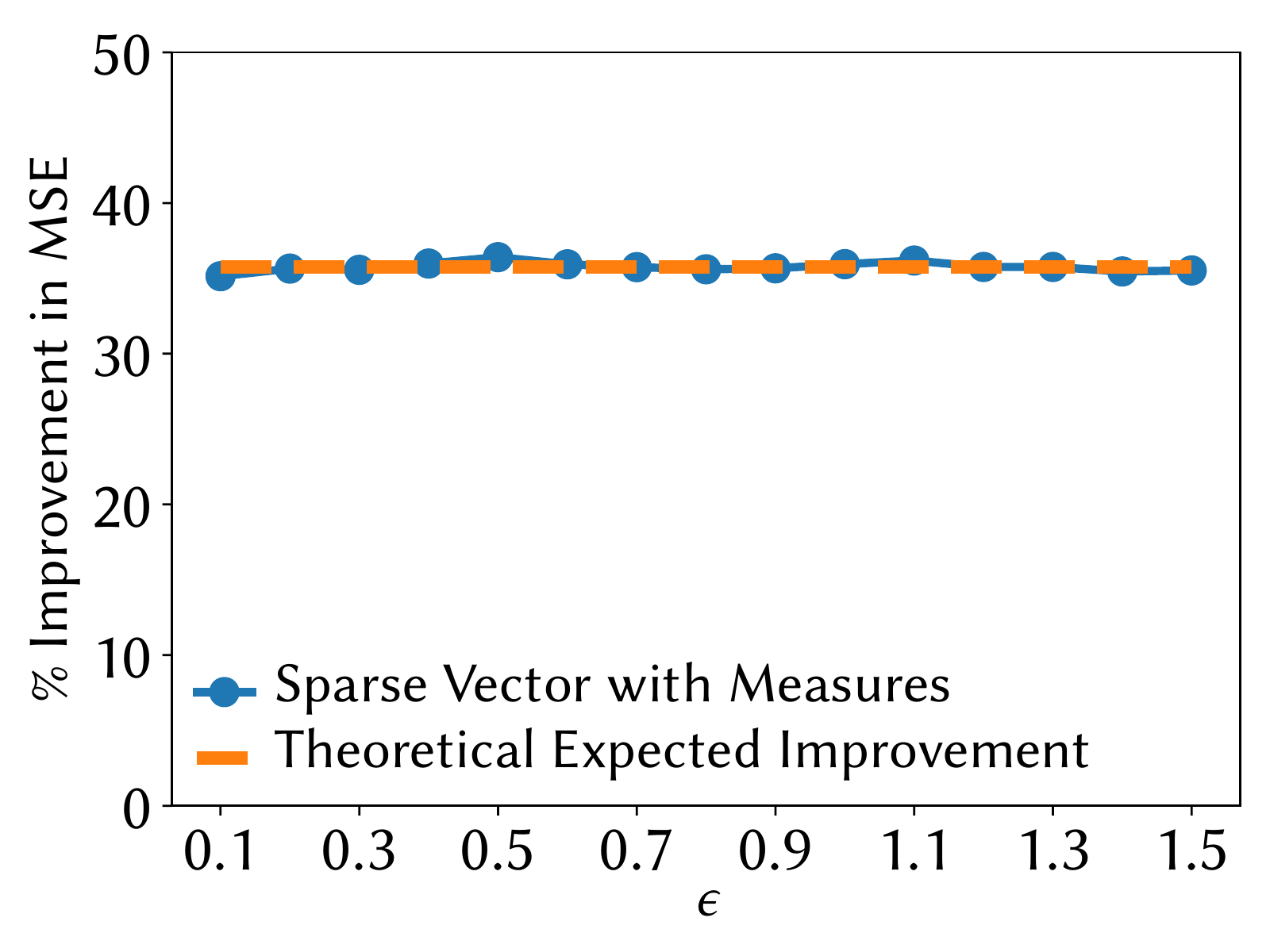} 
\caption{\svtmeasure, kosarak.\label{fig:svt_measure_kosarak_epsilons_counting}}
\end{subfigure}%
\hspace{0.1\textwidth}
\begin{subfigure}[t]{0.4\textwidth}
\captionsetup{width=1.1\linewidth,format=hang}
\includegraphics[width=\linewidth]{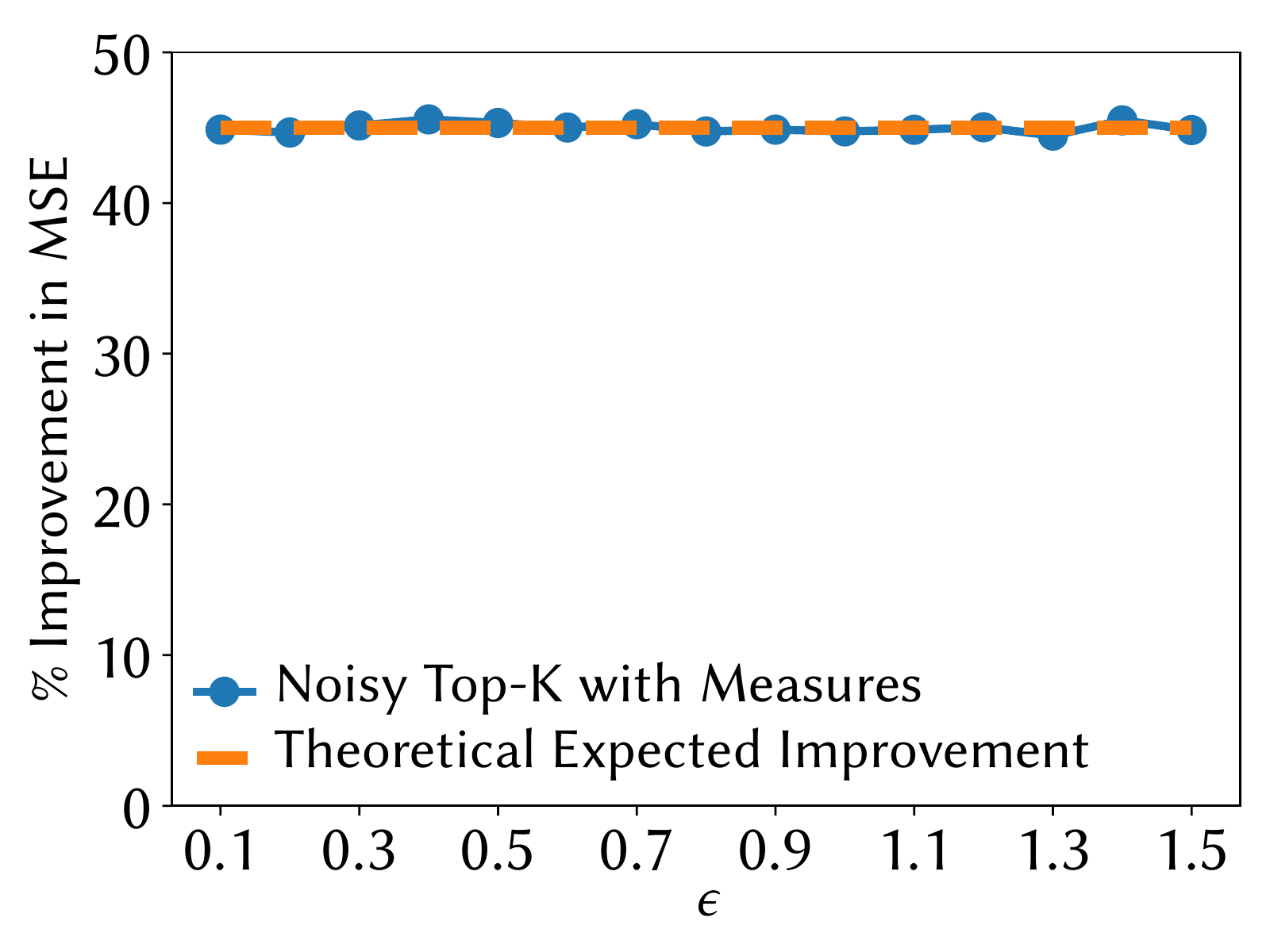} 
\caption{\topkmeasure, kosarak.\label{fig:noisy_topk_measure_kosarak_epsilons_counting}}
\end{subfigure}
\caption{\hl{Improvement percentage of Mean Squared Error on monotonic queries, for different $\epsilon$, for Sparse Vector with Gap and Noisy Top-K with Gap when half the privacy budget is used for query selection and the other half is used for measurement of their answers.  $k$ is set to 10.\label{fig:noisy_topk_measure_epsilons_counting} \label{fig:svt_measure_epsilons_counting}}}
\end{figure*}

\subsection{Datasets}

For evaluation, we used the two real datasets from \cite{lyu2017understanding}: BMP-POS, Kosarak and 
 a synthetic dataset T40I10D100K created by 
 the generator from the IBM Almaden Quest research group. These datasets are collections of transactions (each transaction is a set of items).
 In our experiments, the queries correspond to the counts of each item (i.e. how many transactions contained item \#23?)
 The statistics of the datasets are listed below. 
%
\begin{center}
\begin{tabular}{c c c} 
\Xhline{1.5\arrayrulewidth}
\textbf{Dataset} & \textbf{\# of Records} & \textbf{\# of Unique Items} \\
\hline
BMS-POS & 515,597 & 1,657 \\
Kosarak & 990,002 & 41,270\\
T40I10D100K & 100,000 & 942 \\
\Xhline{1.5\arrayrulewidth}
\end{tabular}
\end{center}




\begin{figure*}[!ht]
\centering
\captionsetup{width=.95\linewidth}
\begin{subfigure}[t]{0.325\textwidth}
\captionsetup{width=.9\linewidth,format=hang}
\includegraphics[width=\linewidth]{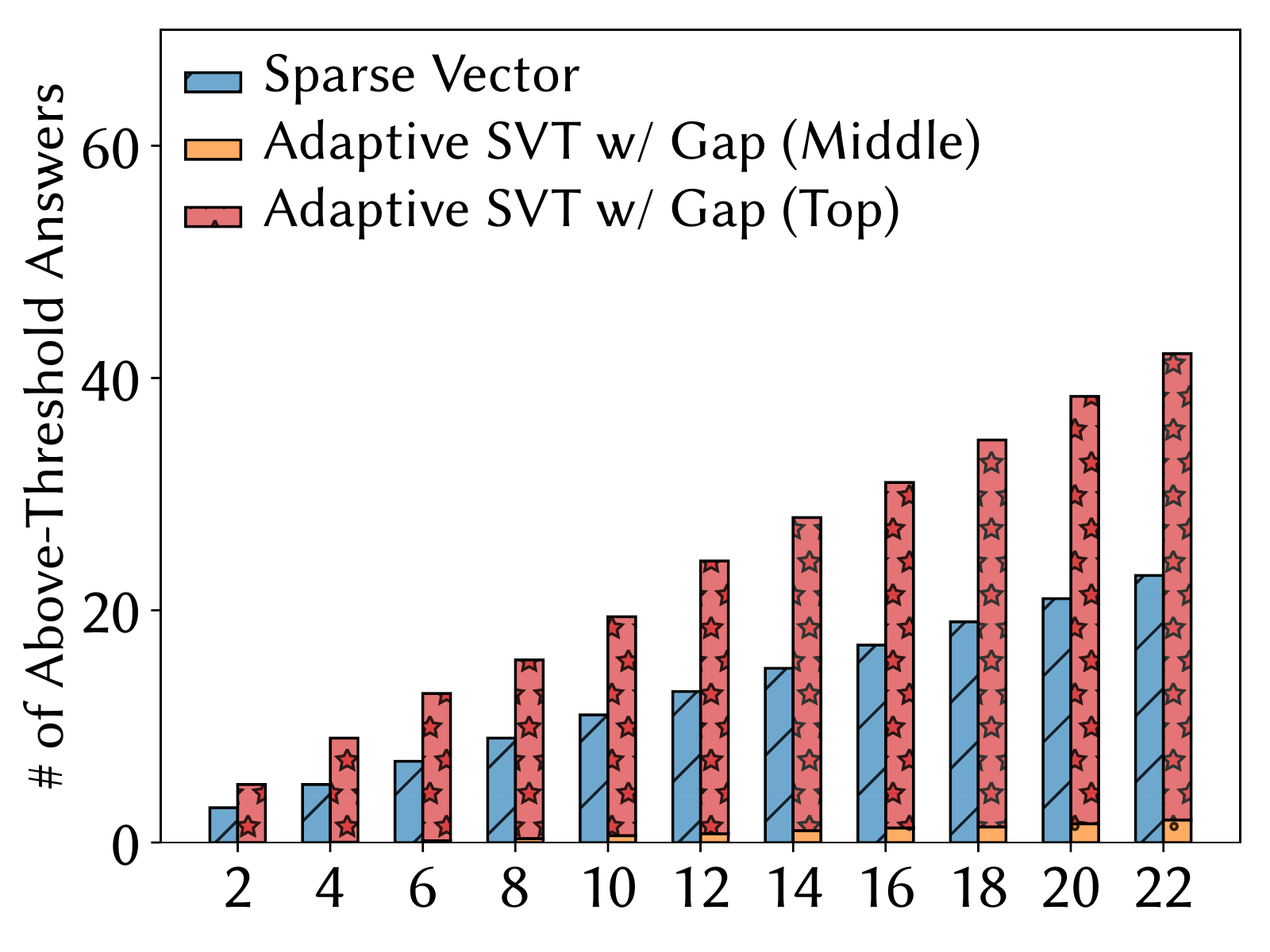}
\caption{\# of queries answered, BMS-POS.\label{fig:adaptive_ata_bms_pos_counting}}
\end{subfigure}%
\begin{subfigure}[t]{0.325\textwidth}
\captionsetup{width=.9\linewidth,format=hang}
\includegraphics[width=\linewidth]{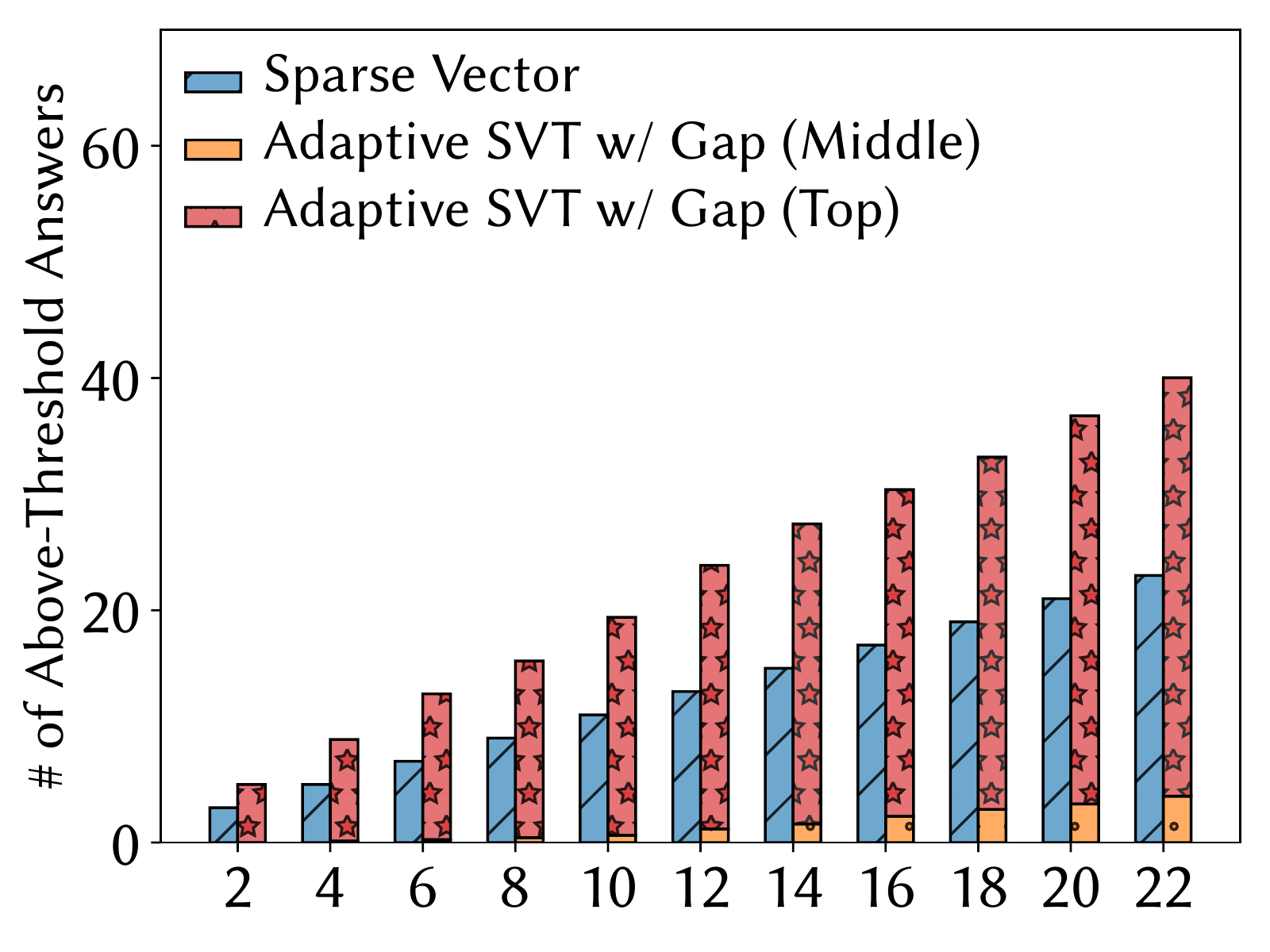}
\caption{\# of queries answered, kosarak.\label{fig:adaptive_ata_kosarak_counting}}
\end{subfigure}
\begin{subfigure}[t]{0.325\textwidth}
\captionsetup{width=.95\linewidth,format=hang}
\includegraphics[width=\linewidth]{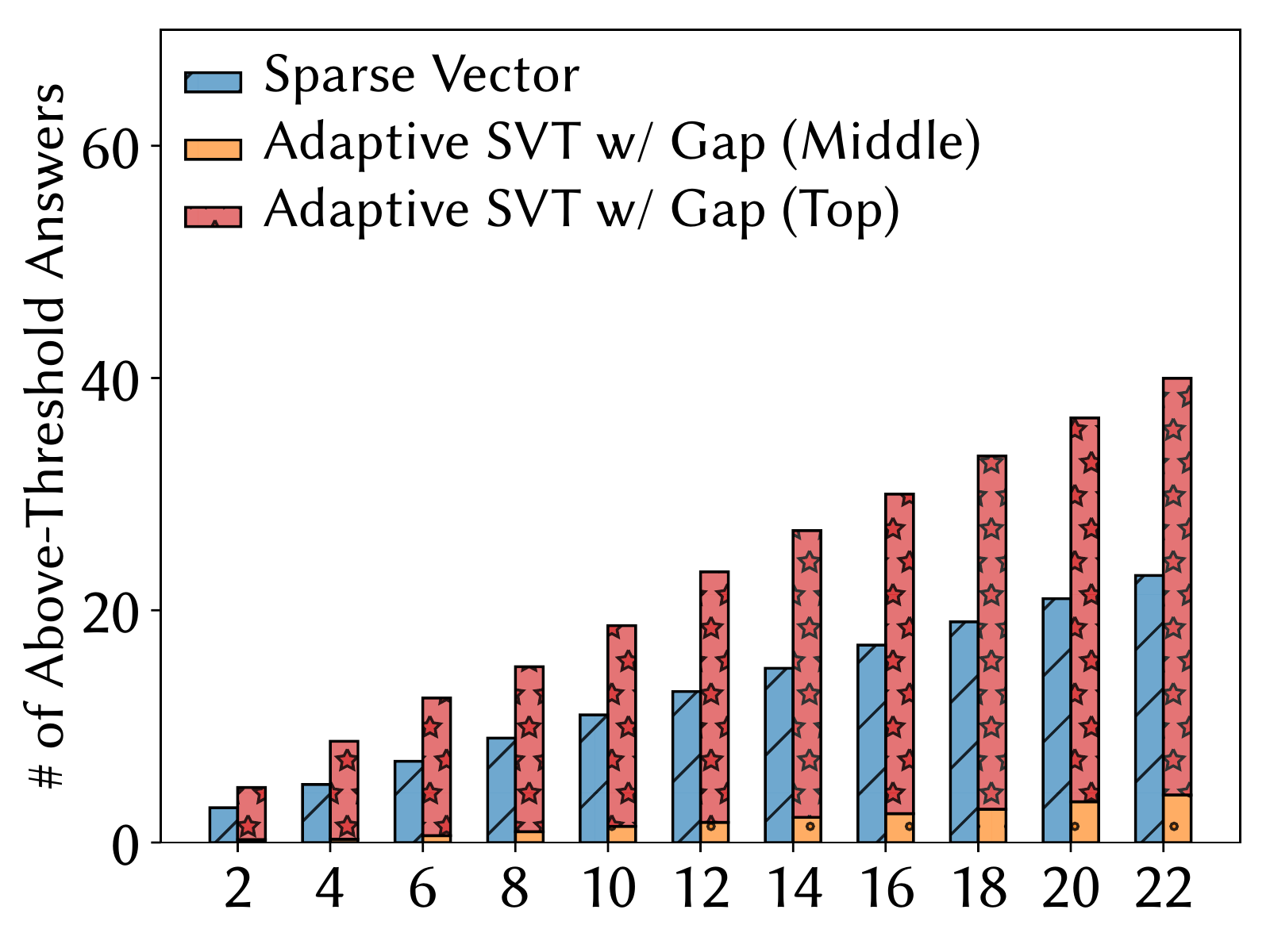}
\caption{\# of queries answered, T40I10D100K.\label{fig:adaptive_ata_t40_counting}}
\end{subfigure}%

\begin{subfigure}[t]{0.33\textwidth}
\captionsetup{width=\linewidth,format=hang}
\includegraphics[width=\linewidth]{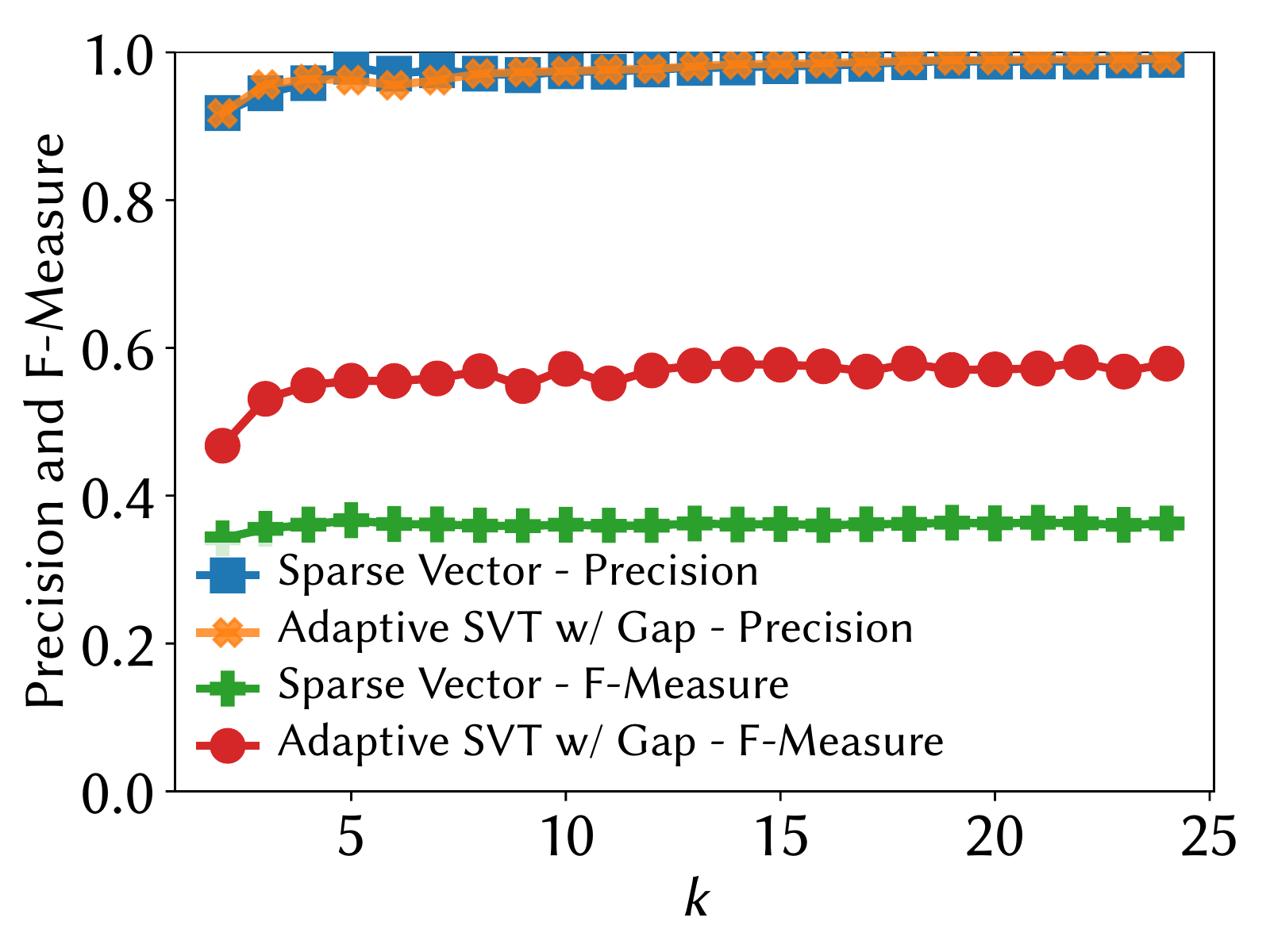}
\caption{Precision and F-Measure, BMS-POS.\label{fig:adaptive_precision_bms_pos_counting}}
\end{subfigure}%
\begin{subfigure}[t]{0.33\textwidth}
\captionsetup{width=\linewidth,format=hang}
\includegraphics[width=\linewidth]{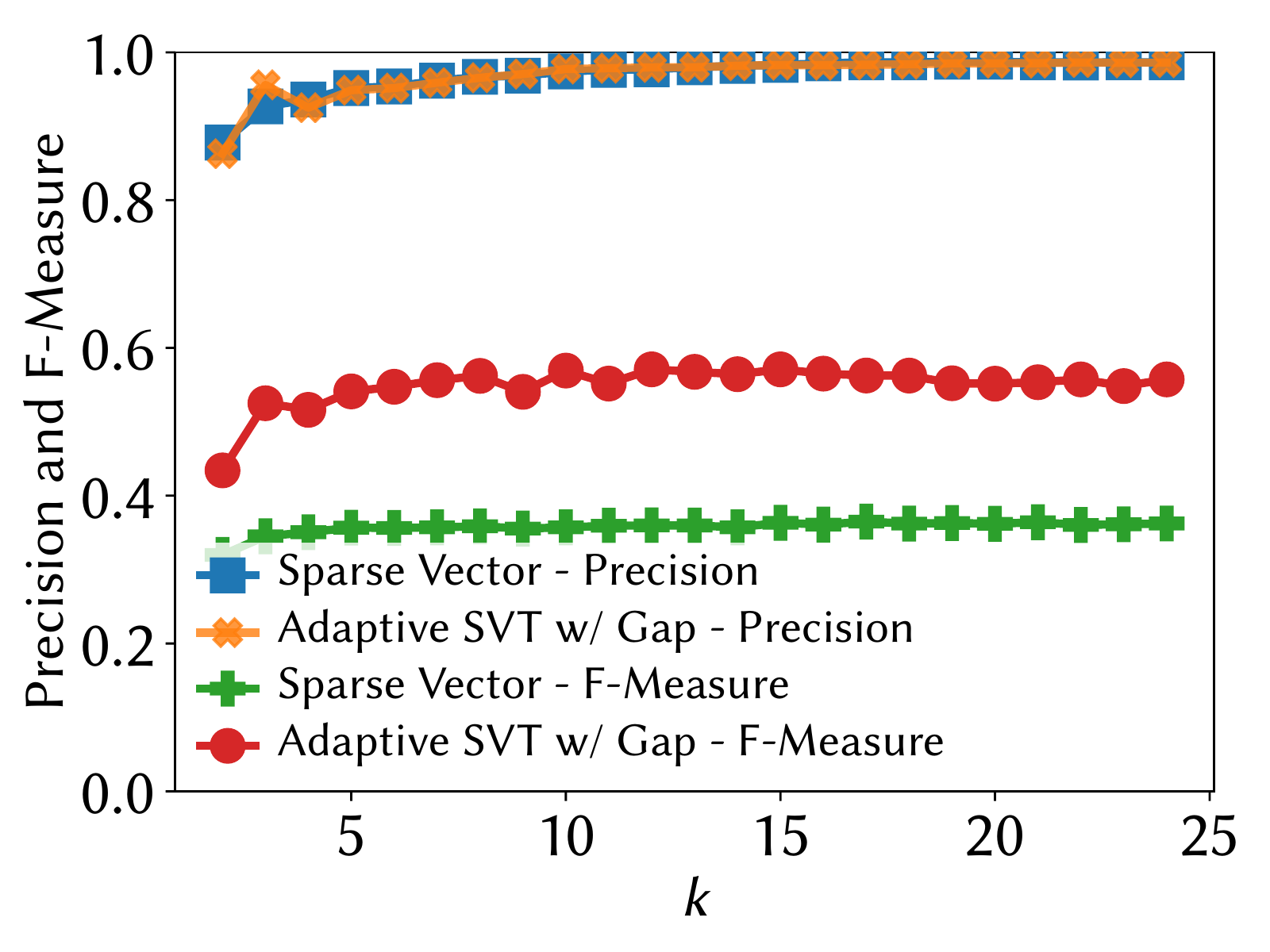}
\caption{Precision and F-Measure, kosarak.\label{fig:adaptive_precision_kosarak_counting}\\}
\end{subfigure}%
\begin{subfigure}[t]{0.33\textwidth}
\captionsetup{width=\linewidth,format=hang}
\includegraphics[width=\linewidth]{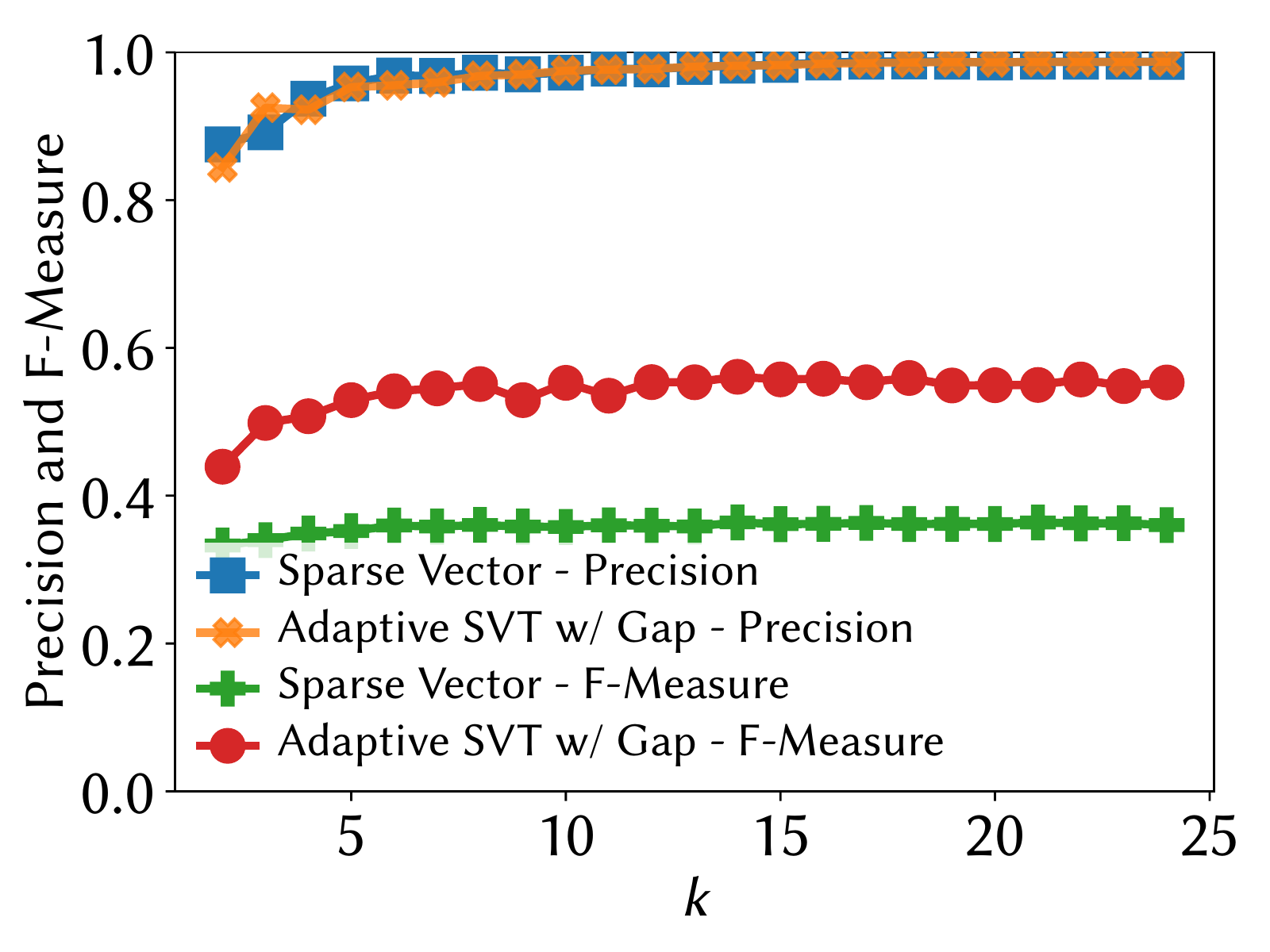}
\caption{Precision and F-Measure, T40I10D100K.\label{fig:adaptive_precision_t40_counting}}
\end{subfigure}
\caption{Results for \adaptivesvt under different $k$'s for monotonic queries. Privacy budget $\epsilon = 0.7$ and $x$-axis: $k$.\label{fig:adaptive_counting}}
\end{figure*}

\begin{figure}[!ht]
\centering
\includegraphics[scale=0.45]{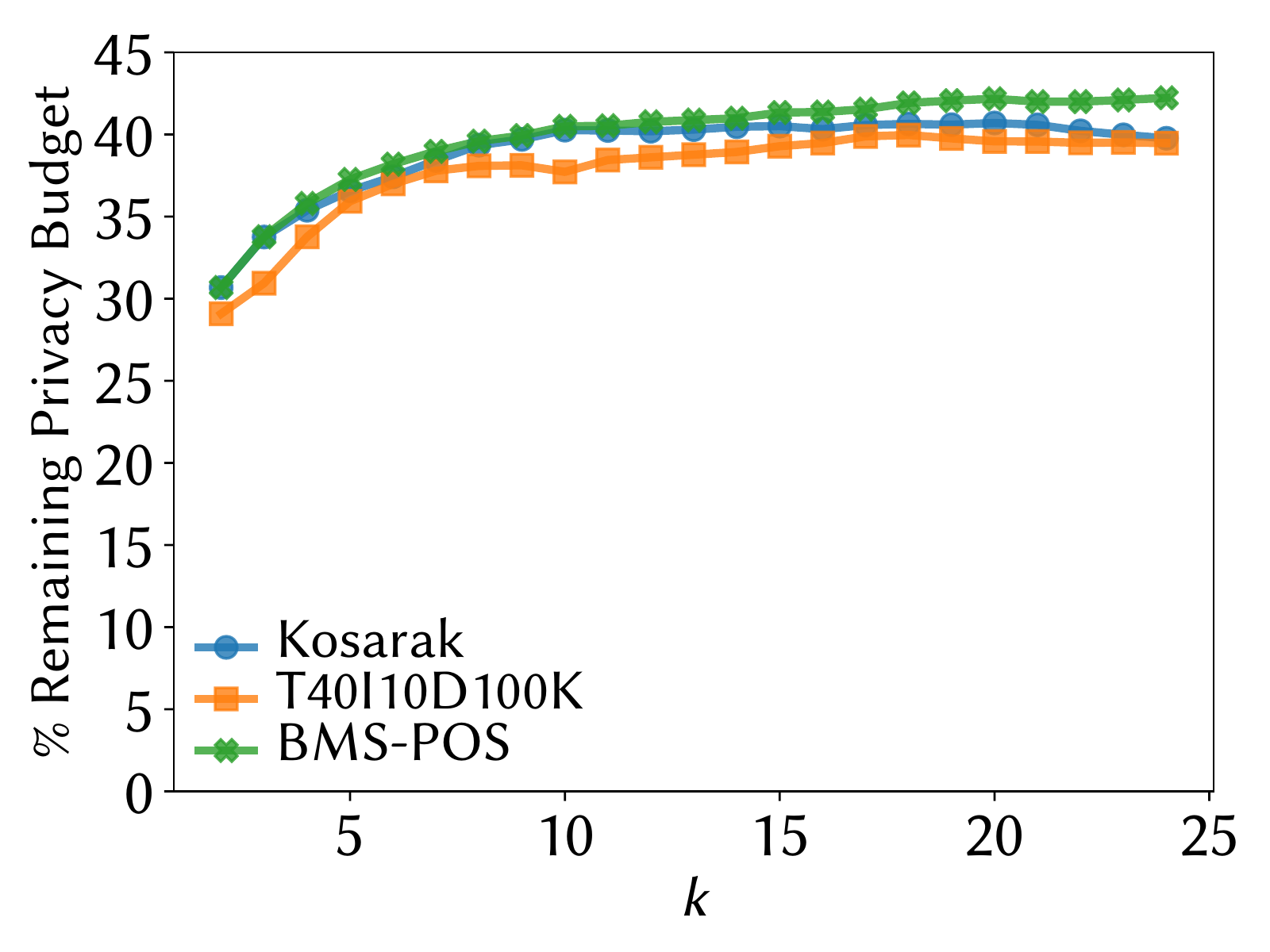}
\caption{\hl{Remaining privacy budget when \adaptivesvt is stopped after answering $k$ queries using different datasets. Privacy budget $\epsilon = 0.7$.\label{fig:adaptive_budget_kosarak_counting}}}
\end{figure}

\subsection{Gap Information + Postprocessing}\label{sec:postprocessing}
The first set of experiments is to measure how gap information can help us improve estimates in selected queries. We use the setup of Sections \ref{sec:blue} and \ref{sec:gapsvt_measures}. That is, a data analyst splits the privacy budget $\epsilon$ in half. She uses the first half to select $k$ queries -- using \gaptopk or {\gapsvt} \hl{ (or \adaptivesvt)} -- and then uses the second half of the privacy budget to obtain independent noisy measurements of each selected query. 

If one were unaware that gap information came for free, one would just use those noisy measurements as estimates for the query answers. The error of this approach is the gap-free baseline. However, since the gap information does come for free, we can use the postprocessing described in Sections \ref{sec:blue} and \ref{sec:gapsvt_measures} to improve accuracy (we call this latter approach \gapsvt with Measures and \gaptopk with Measures).

We first evaluate the percent improvement in mean squared error (MSE) of the postprocessing approach compared to the gap-free baseline and compare this improvement to our theoretical analysis.
As discussed in Section~\ref{sec:gapsvt_measures}, we set the budget allocation ratio \emph{within} the \gapsvt algorithm (i.e., the budget allocation between the threshold and queries) to be \hl{$1:k^\frac{2}{3}$ for monotonic queries and} $1:(2k)^\frac{2}{3}$ \hl{otherwise}  -- such a ratio is recommended in \cite{lyu2017understanding} for the original \svt. \hl{The threshold used for \gapsvt is randomly picked from the top $2k$ to top $8k$} in each dataset \hl{for each run}.\footnote{\hl{Selecting thresholds for SVT in experiments is difficult, but we feel this may be fairer than averaging the answer to the top $k^\text{th}$ and $k+1^\text{th}$ queries as was done in prior work  \cite{lyu2017understanding}.}} All numbers plotted are averaged over $10,000$ runs. \hl{Due to space constraints, we only show experiments for counting queries (which are monotonic).}

Our theoretical analysis in Sections \ref{sec:blue} and \ref{sec:gapsvt_measures} suggested that the improvements can reach up to \hl{50\% in case of monotonic queries (and 20\% for non-monotonic queries)} as $k$ increases. This is confirmed in Figures 
\hl{\ref{fig:svt_measure_bms_pos_counting}, 
}
for \gapsvt and Figures 
\hl{\ref{fig:noisy_topk_measure_bms_pos_counting}, 
}for our Top-K algorithm \hl{using the  BMS-POS dataset (results for the other datasets are nearly identical)}. These figures plot the theoretical and empirical percent improvement in MSE as a function of $k$ and show the power of the free gap information. 

We also generated corresponding plots where $k$ is held fixed and the total privacy budget $\epsilon$ is varied. \hl{We only present the result for the kosarak dataset as results for the other datasets are nearly identical.} For \gapsvt,  
\hl{Figures 
\ref{fig:svt_measure_kosarak_epsilons_counting}}
confirms that this improvement is stable for different $\epsilon$ values. For our Top-K algorithm, 
\hl{Figures 
\ref{fig:noisy_topk_measure_kosarak_epsilons_counting}}
confirms that this improvement is also stable for different values of $\epsilon$.

\subsection{Benefits of Adaptivity}
In this subsection we present an evaluation of \hl{the budget-saving properties of} our novel \adaptivesvt algorithm to show that it can answer more above-threshold queries than \svt and \gapsvt at the same privacy cost \hl{(or, conversely, answer the same number of queries but with leftover budget that can be used for other purposes)}. First note that \svt and \gapsvt both answer exactly the same amount of queries, so we only need to compare \adaptivesvt to the original \svt \cite{diffpbook,lyu2017understanding}. In both algorithms, the budget allocation between the threshold noise and query noise is set according to the ratio \hl{$1:k^\frac{2}{3}$ (i.e., the hyperparameter $\theta$ in \adaptivesvt is set to $1/(1+k^\frac{2}{3})$)}, following recommendations for SVT by Lyu et. al.  \cite{lyu2017understanding}.
\hl{The threshold is randomly picked from the top $2k$ to top $8k$} in each dataset and all reported numbers are averaged over $10,000$ runs.

\paragraph*{Number of queries answered} We first compare the number of queries answered by each algorithm as the parameter $k$ is varied from 2 to 25 with a privacy budget of $\epsilon=0.7$ (results for other settings of the total privacy budget are similar). The results are shown in Figure \ref{fig:adaptive_ata_bms_pos_counting}, \ref{fig:adaptive_ata_kosarak_counting}, and \ref{fig:adaptive_ata_t40_counting}. In each of these bar graphs, the left (blue) bar is the number of answers returned by \svt and the right bar is the number of answers returned by \adaptivesvt. This right bar is broken down into two components: the number of queries returned from the top ``if'' branch (corresponding to queries that were significantly larger than the threshold even after a lot of noise was added) and the number of queries returned from the middle ``if'' branch. Queries returned from the top branch of \adaptivesvt have less privacy cost than the queries returned by \svt. Queries returned from the middle branch of \adaptivesvt have the same privacy cost as in \svt.
%
We see that most queries are answered in the top branch of \adaptivesvt, meaning that the above-threshold queries were generally large (much larger than the threshold). Since \adaptivesvt uses more noise in the top branch, it uses less privacy budget to answer those queries and uses the remaining budget to provide additional answers (up to an average of 15 more answers when $k$ was set to 25). 



\paragraph*{Precision and F-Measure}
Although the adaptive algorithm can answer more above-threshold queries than the original, one can still ask the question of whether the returned queries really are above the threshold. Thus we can look at the precision of the returned results (the fraction of returned queries that are actually above the threshold) and \hl{the widely used F-Measure (the harmonic mean of precision and recall)}. One would expect that the precision of \adaptivesvt should be less than that of \svt, because the adaptive version can use more noise when processing queries. In Figures \ref{fig:adaptive_precision_bms_pos_counting}, \ref{fig:adaptive_precision_kosarak_counting}, and \ref{fig:adaptive_precision_t40_counting} we compare the precision and F-Measure of the two algorithms. Generally we see very little difference in precision. \hl{On the other hand, since \adaptivesvt answers more queries while maintaining high precision, the recall of \adaptivesvt would be much larger than \svt, thus leading to the F-Measure being roughly 1.5 times that of \svt. }


\hl{\paragraph*{Remaining Privacy Budget} If a query is large, \adaptivesvt may only need to use a small part of the privacy budget to determine that the query is likely above the noisy threshold. That is, it may produce an output in its top branch, where a lot of noise (hence less privacy budget) is used. If we stop \adaptivesvt after $k$ returned queries, it may still have some privacy budget left over (in contrast to standard versions of Sparse Vector, which use up all of their privacy budget). This remaining privacy budget can then be used for other data analysis tasks. 
For all three datasets, Figure \ref{fig:adaptive_budget_kosarak_counting} 
 shows the percentage of privacy budget that is left over when \adaptivesvt is run with parameter $k$ and stopped after $k$ queries are returned. We see that roughly 40\% of the privacy budget is left over, confirming that \adaptivesvt is able to save a significant amount of privacy budget.

}

\section{General Randomness Alignment and Proof of Lemma \lowercase{\ref{lem:alignmentbound}}}\label{subsec:incproofs}
In this section, we prove Lemma \ref{lem:alignmentbound}, which was used to establish the privacy properties of the algorithms we proposed. The proof of the lemma requires a more general theorem for working with randomness alignment functions. We explicitly list all of the conditions needed for the sake of reference (many prior works had incorrect proofs because they did not have such a list to follow).
In the general setting, the method of randomness alignment requires the following steps.
\begin{enumerate}[leftmargin=0.5cm,itemsep=0cm,topsep=0.5em,parsep=0.5em]
    \item For each pair of adjacent databases $D\sim D^\prime$ and $\omega\in\Omega$, define a randomness alignment $\ali$ or local alignment functions $\alio: \sdo\rightarrow\sdop$ (see notation in Table \ref{tab:notation}). In the case of local alignments this involves proving that if $M(D,H)=\omega$ then $M(D^\prime, \alio(H))=\omega$.

    \item Show that $\ali$ (or all the $\alio$) is  one-to-one (it does not need to be onto). That is, if we know $D,D^\prime,\omega$ and we are given the value $\ali(H)$ (or $\alio(H)$), we can obtain the value $H$. 
    
    \item For each pair of adjacent databases $D\sim D^\prime$, bound the \emph{alignment cost} of $\ali$ ($\ali$ is either given or constructed by piecing together the local alignments). Bounding the alignment cost means the following: If $f$ is the density (or probability mass) function of $H$, find a constant $a$ such that $f(H)/f(\ali(H))\leq a$ for all $H$ (except a set of measure 0). In the case of local alignments, one can instead  show the following. For all $\omega$, and adjacent $D\sim D^\prime$ the ratio $f(H)/f(\alio(H))\leq a$ for all $H$ (except on a set of measure 0). 
    
    \item Bound the \emph{change-of-variables cost} of $\ali$ (only necessary when $H$ is not discrete). One must show that the Jacobian  of $\ali$, defined as $J_{\ali} =\frac{\partial \ali}{\partial H}$, exists (i.e. $\ali$ is differentiable) and is continuous except on a set of measure 0. Furthermore,  for all pairs $D\sim D^\prime$, show the quantity $\abs{\det J_{\ali}}$ is lower bounded by some constant $b > 0$. If $\ali$ is constructed by piecing together local alignments $\alio$ then this is equivalent to showing the following (i) $\abs{\det J_{\alio}}$ is lower bounded by some constant $b > 0$ for every $D\sim D^\prime$ and $\omega$; and (ii) for each $D\sim D^\prime$, the set $\Omega$ can be partitioned into countably many disjoint measurable sets $\Omega=\bigcup_i \Omega_i$ such that whenever $\omega$ and $\omega^*$ are in the same partition, then $\alio$ and $\alios$ are the same function. Note that this last condition (ii) is equivalent to requiring that the local alignments must be defined without using the axiom of choice (since non-measurable sets are not constructible otherwise) and
    for each $D\sim D^\prime$, the number of distinct local alignments is countable. That is, the set $\{\alio\mid \omega\in\Omega\}$ is countable (i.e., for many choices of $\omega$ we get the same exact alignment function).
\end{enumerate}

\begin{theorem}\label{thm:main}
Let $M$ be a randomized algorithm that terminates with probability 1 and suppose the  number  of  random  variables  used  by $M$ can  be  determined from its output. If, for all pairs of adjacent databases $D\sim D^\prime$, there exist randomness alignment functions $\ali$ (or local alignment functions $\alio$ for all $\omega\in\Omega$ and $D\sim D^\prime$) that satisfy conditions 1 though 4 above,  then  $M$ satisfies $\ln(a/b)$-differential privacy.
\end{theorem}
\begin{proof} 
We need to show that for all $D\sim D'$ and $E\subseteq \Omega$, $\PP(\sde) \leq (a/b) \PP(\sdep)$.

First we note that if we have a randomness alignment $\ali$, we can define corresponding local alignment functions as follows $\alio(H)=\ali(H)$ (in other words, they are all the same). The conditions on local alignments are a superset of the conditions on randomness alignments, so for the rest of the proof we work with the $\alio$.

Let $\phi_1, \phi_2, \ldots $ be the distinct local alignment functions (there are countably many of them by Condition 4).
Let $E_i = \set{\omega\in E \mid \alio = \phi_i}$. 
By Conditions 1 and 2 we have that for each $\omega\in E_i$, $\phi_i$ is one-to-one on $\sdo$ and $\phi_i(\sdo)\subseteq \sdop $.
Note that $\sdei = \cup_{\omega\in E_i}\sdo$ and $\sdeip = \cup_{\omega\in E_i}\sdop$. Furthermore, the sets $\sdo$ are pairwise disjoint for different $\omega$ and the sets $\sdop$ are pairwise disjoint for different $\omega$. It follows that $\phi_i$ is one-to-one on $\sdei$ and $\phi_i(\sdei)\subseteq \sdeip$.
Thus for any $H^\prime\in \phi_i(\sdei)$ there exists $H\in \sdei$ such that $H=\phi_i^{-1}(H^\prime)$.
By Conditions 3 and 4, we have $\frac{f(H)}{f(\phi_i(H))} = \frac{f(\phi_i^{-1}(H'))}{f(H')} \leq a$ for all $H\in \sdei$, and $\abs{\det J_{\phi_i} } \geq b$ (except on a set of measure 0).
Then the following is true:
\begin{align*}
    \PP(\sdei) &= \int_{\sdei}f(H)~dH \\
    &= \int_{\phi_i(\sdei)} f(\phi_i^{-1}(H^\prime))~\frac{1}{\abs{\det J_{\phi_i}}}~dH'\\
    &\leq \int_{\phi_i(\sdei)} a f(H')\frac{1}{b}~dH'
    = \frac{a}{b} \int_{\phi_i(\sdei)} f(H')~dH' \\
    &\leq \frac{a}{b} \int_{\sdeip} f(H^\prime)~dH^\prime = \frac{a}{b} \PP(\sdeip).
\end{align*}
The second equation is the change of variables formula in calculus. The last inequality follows from the containment  $\phi_i(\sdei) \subseteq \sdeip$ and the fact that the density $f$ is nonnegative. In the case that $H$ is discrete, simply replace the density $f$ with a probability mass function, change the integral into a summation, ignore the Jacobian term and set $b=1$.
Finally, since $E=\cup_iE_i$ and $E_i\cap E_j =\emptyset$ for $i\ne j$, we conclude that
\[\PP( \sde)=  \sum_i\PP( \sdei) \leq \frac{a}{b} \sum_i \PP( \sdeip) = \frac{a}{b} \PP(\sdep).\]
\end{proof}

We now present the proof of \textbf{Lemma \ref{lem:alignmentbound}.}

\begin{proof} Let $\alio(H)= H'=(\eta'_1, \eta'_2, \ldots)$. By acyclicity there is some permutation $\pi$ under which $\eta_{\pi(1)} = \eta'_{\pi(1)} - c$ where $c$ is some constant depending on $D\sim D'$ and $\omega$. Thus $\eta_{\pi(1)}$ is uniquely determined by $H'$. Now (as an induction hypothesis) assume $\eta_{\pi(1)}, \ldots, \eta_{\pi(j-1)}$ are uniquely determined by $H'$ for some $j > 1$, then $\eta_{\pi(j)} = \eta_{\pi(j)}'-\alpsi{j}(\eta_{\pi(1)}, \ldots, \eta_{\pi(j-1)})$, so $\eta_{\pi(j)}$ is also uniquely determined by $H'$. Thus by strong induction $H$ is uniquely determined by $H'$, i.e., $\alio$ is one-to-one. 
It is easy to see that with this ordering, $J_{\alio}$ is an upper triangular matrix with $1$'s on the diagonal. Since permuting variables doesn't change  $\abs{\det J_{\alio}}$, we have $\abs{\det J_{\alio}} = 1$ since that is the determinant of upper triangular matrices.
\hl{Furthermore, (recalling the definition of the cost of $\alio$), clearly }
\[
\hl{\ln \frac{f(H)}{f(\phi_\omega(H))} = \sum_i \ln \frac{f_i(\eta_i)}{f_i(\eta_i^\prime)} \leq \sum_i | \eta_i-\eta^\prime_i|/\alpha_i \leq \epsilon}
\]
\hl{The first inequality follows from Condition \ref{conditioniii} of Lemma \ref{lem:alignmentbound} and the second from Condition \ref{conditioniv}.}
\end{proof}

\section{Conclusions and Future Work}\label{sec:conc}
In this paper we introduced the Adaptive Sparse Vector with Gap and Noisy Top-K with Gap mechanisms, which were based on the observation that the classical Sparse Vector and Noisy Max mechanisms could release additional information at no cost to privacy.  We also provided applications of this free gap information.

Future directions include using this technique to design additional mechanisms as well as finding new applications for these mechanisms in fine-tuning the accuracy of data release algorithms that use differential privacy.

\section*{Acknowledgments}
This work was supported by NSF Award CNS-1702760.

\balance


\bibliographystyle{abbrv}
\bibliography{diffpriv}  


\iffullversion
\pagebreak
\begin{appendix}
\section{Proofs}

\subsection{Probability of Ties Among n Queries with the Discrete Laplace Distribution}
A discretized Laplace distribution whose support ranges over multiples of some base $\gamma$ has the following probability mass function:
$f(k; \epsilon) = \frac{1-e^{-\gamma\epsilon}}{1+e^{-\gamma\epsilon}} e^{-\epsilon |k|}$ (for $k=0, \pm \gamma, \pm 2\gamma,\dots$). 

We will first consider the probability of a tie between two queries and then use the union bound over all pairs of queries.

Suppose $\eta_1$ and $\eta_2$ are two i.i.d zero mean discrete Laplace random variables with scale $1/\epsilon$ and base $\gamma$.
Without loss of generality, let $ q_1 - q_2 = m\gamma \geq 0$. Then the probability that $q_1+\eta_1 = q_2 + \eta_2$ is: 
\begingroup
\allowdisplaybreaks
\begin{align*}
    &\PP(q_1+\eta_1 =q_2+ \eta_2) = \sum_{\ell\in\ZZ} \PP(\eta_1 =  \gamma \ell)\PP(\eta_2 =  (\ell+m)\gamma)\\%
    &=\frac{(1-e^{-\gamma\epsilon})^2}{(1+ e^{-\gamma\epsilon})^2} \sum_{\ell\in \ZZ} e^{-\epsilon\gamma\abs{\ell}} e^{-\epsilon\gamma\abs{\ell+m}}\\%
    &=\frac{(1-e^{-\gamma\epsilon})^2}{(1+ e^{-\gamma\epsilon})^2}  \Big( \sum_{\ell=-\infty}^{-m} e^{\epsilon\gamma \ell}e^{\epsilon\gamma(\ell+m)}\\%
     &\qquad + \sum_{\ell=-m+1}^{0} e^{\epsilon\gamma \ell}e^{-\epsilon\gamma(\ell+m)} + \sum_{\ell=1}^{\infty} e^{-\epsilon\gamma \ell}e^{-\epsilon\gamma(\ell+m)} \Big)\\
    &=\frac{(1-e^{-\gamma\epsilon})^2}{(1+ e^{-\gamma\epsilon})^2} e^{-\epsilon\gamma m} \big(\frac{1}{1-e^{-2\epsilon\gamma}} + m + \frac{e^{-2\epsilon\gamma}}{1-e^{-2\epsilon\gamma}}\big)\\
    &= \frac{(1-e^{-\gamma\epsilon})^2}{(1+ e^{-\gamma\epsilon})^2} e^{-\epsilon\gamma m} (\frac{1+e^{-2\epsilon\gamma}}{1 - e^{-2\epsilon\gamma}} + m)\\
    &\leq  \left(1-e^{-\gamma\epsilon}\right)^2e^{-\epsilon\gamma m}\left(\frac{1}{1-e^{-2\epsilon\gamma}}+m\right)\\
    &\leq  \left(1-e^{-\gamma\epsilon}\right)^2e^{-\epsilon\gamma m}\left(\frac{1}{1-e^{-\epsilon\gamma}}+m\right)\\
    &\leq (1-e^{-\gamma\epsilon}) +  \left(1-e^{-\gamma\epsilon}\right)^2e^{-\epsilon\gamma m}m\\
    &\leq \gamma\epsilon +  \left(1-e^{-\gamma\epsilon}\right)^2e^{-\epsilon\gamma m}m\quad\text{(since $1-e^{-x}\leq x$ for $x\geq -1$)}\\
    &\leq \gamma\epsilon + (\gamma\epsilon)^2me^{-\epsilon\gamma m} =\gamma\epsilon (1 + \gamma\epsilon m e^{-\gamma\epsilon m})\\
    &\leq \gamma\epsilon(1+e^{-1}) \quad\text{(since $xe^{-x}$ is maximized at $x=1$)}
\end{align*}
\endgroup
Since there are $n$ queries, we can conservatively estimate the probability of a tie as the probability that any pair of $n$ items has a tie. Using the union bound, we get the probability of a tie is at most $n^2 \gamma\epsilon$. In floating point, we expect a Laplace distribution to be implemented using a Discrete Laplace with $\gamma$ being close to machine epsilon, which for double-precision floating point numbers is around $2^{-52}$.




\subsection{Proof of Theorem \ref{thm:blue} (BLUE)}
\begin{proof}
Let $q_1, \ldots, q_k$ be the \emph{true} answers to the $k$ queries selected by Noisy-Top-K-with-Gap algorithm. Let $\alpha_i$ be the estimate of $q_i$ using Laplace mechanism, and $g_i$ be the estimate of the gap between $q_i$ and $q_{i+1}$ from Noisy-Top-K-with-Gap.

Recall that $\alpha_i = q_i + \xi_i$ and $g_i = q_i + \eta_i - q_{i+1} - \eta_{i+1}$ where $\xi_i$ and $\eta_i$  are independent Laplacian random variables. Assume without loss of generality that $\var(\xi_i)=\sigma^2$ and $\var(\eta_i) = \lambda\sigma^2$. 
Write in vector notation 
\[
\vq = \begin{bmatrix}q_1 \\ \vdots \\ q_k \end{bmatrix}, 
\vxi = \begin{bmatrix}\xi_1 \\ \vdots \\ \xi_k \end{bmatrix}, 
\veta = \begin{bmatrix}\eta_1 \\ \vdots \\ \eta_k \end{bmatrix}, 
\valpha = \begin{bmatrix}\alpha_1 \\ \vdots \\ \alpha_k \end{bmatrix}, 
\vg = \begin{bmatrix}g_1 \\ \vdots \\ g_{k-1} \end{bmatrix},
\]
then $\valpha = \vq + \vxi$ and $\vg = N(\vq+\veta)$ where
\[N = \begin{bmatrix}
\makebox[1em]{$1$} & \makebox[1em]{$-1$} & \makebox[1em]{} & \makebox[1em]{}\\
\makebox[1em]{} & \makebox[1em]{$\ddots$} & \makebox[1em]{$\ddots$}& \makebox[1em]{}\\
\makebox[1em]{} & \makebox[1em]{} &\makebox[1em]{$1$} & \makebox[1em]{$-1$} 
\end{bmatrix}_{(k-1)\times k}.\]

Our goal is then to find the \emph{best linear unbiased estimate} (BLUE) $\vbeta$ of $\vq$ in terms of $\valpha$ and $\vg$. In other words, we need to find a $k\times k$ matrix $X$ and a $k\times (k-1)$ matrix $Y$ such that 
\begin{equation}\label{eq:blue1}
  \vbeta =X\valpha + Y\vg 
\end{equation}
 with $E(\norm{\vbeta - \vq}^2) $ as small as possible. Unbiasedness implies that $\forall \vq, E(\vbeta) =  X\vq + YN\vq = \vq$. Therefore $X+YN = I_k$ and thus
\begin{equation}\label{eq:blue3}
  X = I_k - YN.  
\end{equation} 
 Plugging this into \eqref{eq:blue1}, we have $\vbeta = (I_k - YN)\valpha + Y\vg = \valpha -Y(N\valpha - \vg)$. Recall that $\valpha = \vq + \vxi$ and $\vg = N(\vq+\veta)$, we have $N\valpha -\vg = N(\vq + \vxi - \vq - \veta) = N(\vxi - \veta)$. Thus
 \begin{equation}\label{eq:blue2}
     \vbeta = \valpha - YN(\vxi - \veta).
 \end{equation}
Write $\vtheta = N(\vxi - \veta)$,  
then we have $ \vbeta - \vq = \valpha - \vq  - Y\vtheta = \vxi - Y\vtheta$. Therefore, finding the BLUE is equivalent to solving the  optimization problem $Y = \arg\min \Phi$ where
\begin{align*}
    \Phi &=  E(\norm{\vxi - Y\vtheta}^2) =   E((\vxi - Y\vtheta)^T(\vxi - Y\vtheta))\\
    &=  E(\vxi^T\vxi - \vxi^TY\vtheta -\vtheta^TY^T\vxi + \vtheta^TY^TY\vtheta)
\end{align*}
Taking the partial derivatives of $\Phi$ w.r.t $Y$, we have 
\begin{align*}
    \frac{\partial \Phi}{\partial Y} &= E(\boldsymbol{0} - \vxi\vtheta^T -\vxi\vtheta^T + Y(\vtheta\vtheta^T + \vtheta\vtheta^T))
\end{align*}
By setting $\frac{\partial \Phi}{\partial Y} = 0$ we have $YE(\vtheta\vtheta^T) = E(\vxi\vtheta^T)$ thus
\begin{equation}\label{eq:blue4}
Y = E(\vxi\vtheta^T) E(\vtheta\vtheta^T)^{-1}.
\end{equation}
Recall that $(\vxi\vtheta^T)_{ij} = \xi_i(\xi_j -\xi_{j+1} -\eta_j + \eta_{j+1} )$, we have 
\[E(\vxi\vtheta^T)_{ij} = \begin{cases}
E(\xi_i^2) = \var(\xi_i) = \sigma^2  & i = j \\
-E(\xi_i^2) = -\var(\xi_i) = -\sigma^2  & i = j+1 \\
0 &\text{otherwise}
\end{cases}
\]
Hence
\[ E(\vxi\vtheta^T) = \sigma^2\begin{bmatrix}
1 & & \\
-1 & \ddots & \\
& \ddots &1 \\
& & -1
\end{bmatrix}_{k\times(k-1)} =  \sigma^2N^T.\]
Similarly, we have 
\begin{align*}
    (\vtheta\vtheta^T)_{ij} &= (\xi_i-\xi_{i+1} - \eta_i + \eta_{i+1})(\xi_j - \xi_{j+1} - \eta_j + \eta_{j+1}) \\
    &= \xi_i\xi_j + \xi_{i+1}\xi_{j+1} - \xi_{i}\xi_{j+1} -\xi_{i+1}\xi_{j} \\
    &\hspace{1em}+\eta_i\eta_j + \eta_{i+1}\eta_{j+1} - \eta_{i}\eta_{j+1} -\eta_{i+1}\eta_{j} \\
    &\hspace{1em}-(\xi_i -\xi_{i+1})(\eta_j-\eta_{j+1}) -(\eta_i -\eta_{i+1})(\xi_j-\xi_{j+1})
\end{align*}
Thus
\[E(\vtheta\vtheta^T)_{ij}\! =\! \begin{cases}
E(\xi_i^2\! +\! \xi_{i+1}^2\! +\! \eta_i^2\! +\! \eta_{i+1}^2) = 2(1\!+\!\lambda)\sigma^2  & i = j \\
E(-\xi_i^2 - \eta_i^2) = -(1\!+\!\lambda)\sigma^2  & i = j\!+\!1 \\
E(-\xi_j^2 -\eta_j^2) = -(1\!+\!\lambda)\sigma^2  & i = j\!-\!1 \\
0 &\text{otherwise}
\end{cases}
\]
Hence
\[
E(\vtheta\vtheta^T) =(1\!+\!\lambda)\sigma^2 \begin{bmatrix}
2 & -1 & & & &\\
-1 & 2 & -1 & & &\\
& \ddots & \ddots &  \ddots &  & \\
& & -1 & 2 & -1 \\
& & & -1 & 2 &
\end{bmatrix}_{(k-1)\times (k-1)}.
\]
It can be directly computed that $E(\vtheta\vtheta^T)^{-1} $ is a symmetric matrix whose lower trianguilar part is 
\[ 
 \frac{1}{k(1\!+\!\lambda)\sigma^2}\!\begin{bmatrix}
(k\!-\!1)\cdot 1 & \cdots & \cdots & \cdots & \cdots   \\
(k\!-\!2)\cdot 1 & (k\!-\!2)\cdot 2 & \cdots & \cdots & \cdots \\
(k\!-\!3)\cdot 1 & (k\!-\!3)\cdot 2 & (k\!-\!3)\cdot 3 & \cdots & \cdots  \\
\vdots & \vdots & \vdots & \ddots &  \vdots \\
1\cdot 1 & 1\cdot 2 & 1\cdot 3 & \cdots & 1\cdot (k\!-\!1)  \\
\end{bmatrix}
\]
i.e., $E(\vtheta\vtheta^T)^{-1}_{ij} = E(\vtheta\vtheta^T)^{-1}_{ji} = \frac{1}{k(1+\lambda)\sigma^2}\cdot (k-i)\cdot j$ for all $1\leq i\leq j \leq k-1$. 
Therefore, $Y = E(\vxi\vtheta^T) E(\vtheta\vtheta^T)^{-1}= $ 
\[
\frac{1}{k(1\!+\!\lambda)}\left(\begin{bmatrix}
k\!-\!1 & k\!-\!2 & \cdots &1 \\
k\!-\!1 & k\!-\!2 & \cdots &1 \\
k\!-\!1 & k\!-\!2 & \cdots &1 \\
\vdots & \vdots & \ddots &\vdots\\
k\!-\!1 & k\!-\!2 & \cdots & 1 \\
\end{bmatrix} - 
\begin{bmatrix}
0 & 0 & \cdots &0 \\
k & 0 & \cdots &0 \\
k & k & \cdots &0 \\
\vdots & \vdots & \ddots &0 \\
k & k & \cdots &k \\
\end{bmatrix}
\right)_{k\times (k-1)}
\]
Hence 
\[
X = I_k - YN = \frac{1}{k(1\!+\!\lambda)}\begin{bmatrix}
1\!+\!k\lambda  & 1 & \cdots & 1 \\
1 & 1\!+\!k\lambda  & \cdots & 1 \\
\vdots & \vdots & \ddots & \vdots \\
1 & 1 & \cdots & 1\!+\!k\lambda  \\
\end{bmatrix}_{k\times k}.
\]
\end{proof}

\subsection{Proof of Corollary \ref{cor:blue_var}}
Recall that $\alpha_i = q_i + \xi_i$ and $g_i = q_i + \eta_i - q_{i+1} - \eta_{i+1}$ where $\xi_i$ and $\eta_i$  are independent Laplacian random variables. Assume without loss of generality that $\var(\xi_i)=\sigma^2$ and $\var(\eta_i) = \lambda\sigma^2$ as before.
From the matrices $X$ and $Y$ in Theorem~\ref{thm:blue} we have that $\beta_i = \frac{x_i + y_i}{k(1+\lambda)}$ where
\begin{align*}
    x_i &= \alpha_1 + \cdots + (1+k\lambda) \alpha_i + \cdots + \alpha_k\\
    &= (q_1+\xi_1) + \cdots + (1+k\lambda) (q_i+\xi_i) + \cdots + (q_k +\xi_k)
    \shortintertext{and}
    y_i &= -g_1 -2g_2 - \cdots - (i-1)g_{i-1} \\
    &\qquad + (k-i) g_i + \ldots + 2g_{k-2} + g_{k-1}\\
    &= -(q_1 + \eta_1) - (q_2+\eta_2) -\cdots -  (q_{i-1} + \eta_{i-1}) \\
    &\qquad + (k-1)(q_i + \eta_i) - (q_{i+1} + \eta_{i+1}) - \cdots - (q_{k} + \eta_k).
\end{align*}
Therefore
\begin{align*}
    \var(x_i) &= \sigma^2 + \cdots + (1+k\lambda)^2\sigma^2 + \cdots + \sigma^2\\
    &= ( k^2\lambda^2 + 2k\lambda + k)\sigma^2 \\
    \var(y_i) &= \lambda\sigma^2+ \cdots + (k-1)^2 \lambda\sigma^2 + \cdots + \lambda\sigma^2\\
    &= (k^2 - k)\lambda\sigma^2
    \shortintertext{and thus}
    \var(\beta_i) &= \frac{\var(x_i) + \var(y_i)}{k^2(1+\lambda)^2} =\frac{1 + k\lambda}{k+k\lambda}\sigma^2.
\end{align*}
Since $\var(\alpha_i) = \var(\xi_i) = \sigma^2$, we have 
\[\frac{\var(\beta_i)}{\var(\alpha_i)} = \frac{1 + k\lambda}{k+k\lambda}.\]

\subsection{Proof of Lemma~\ref{lem:confint}}
The density function of $\eta_i - \eta$ is 
\begin{align*}
    f_{\eta_i-\eta} (z) &= \int_{-\infty}^\infty  f_{\eta_i}(x) f_{\eta}(x-z)~dx \\
    &=\frac{\epsilon_0\epsilon_*}{4} \int_{-\infty}^\infty e^{-\epsilon_*\abs{x}} e^{-\epsilon_0\abs{x-z}}~dx.
\end{align*}
First  consider the case $\epsilon_0\neq \epsilon_*$. When $z\geq 0$, we have
\begin{align*}
     f_{\eta_i-\eta} (z) &= \frac{\epsilon_0\epsilon_*}{4} \int_{-\infty}^\infty e^{-\epsilon_*\abs{x}} e^{-\epsilon_0\abs{x-z}}~dx\\
    &= \frac{\epsilon_0\epsilon_*}{4} \Big(  \int_{-\infty}^0 e^{\epsilon_* x} e^{\epsilon_0(x-z)}~dx ~+ \\
    &\quad \int_{0}^z e^{-\epsilon_* x} e^{\epsilon_0(x-z)}~dx
     + \int_{z}^\infty e^{-\epsilon_* x} e^{-\epsilon_0(x-z)}~dx\Big)\\ %
    &= \frac{\epsilon_0\epsilon_*}{4} \Big( \frac{e^{-\epsilon_0z}}{\epsilon_0+\epsilon_*} + \frac{ e^{-\epsilon_*z} - e^{-\epsilon_0z}}{\epsilon_0-\epsilon_*} + \frac{e^{-\epsilon_*z}}{\epsilon_0+\epsilon_*} \Big)\\ %
    &= \frac{\epsilon_0\epsilon_* (\epsilon_0 e^{-\epsilon_*z} - \epsilon_* e^{-\epsilon_0z})}{2(\epsilon_0^2 - \epsilon_*^2)}
\end{align*}
Thus by symmetry we have for all $z\in \RR$
\[f_{\eta_i-\eta} (z) = \frac{\epsilon_0\epsilon_* (\epsilon_0 e^{-\epsilon_*\abs{z}} - \epsilon_* e^{-\epsilon_0\abs{z}})}{2(\epsilon_0^2-\epsilon_*^2)} \] and 
\begin{align*}
    \PP(\eta_i - \eta \geq -t ) &= \int_{-t}^{\infty} f_{\eta_i-\eta} (z)~dz =  \int_{-t}^{0} f_{\eta_i-\eta} (z)~dz  + \frac{1}{2} \\
    &= 1 - \frac{\epsilon_0^2 e^{-\epsilon_*t} - \epsilon_*^2e^{-\epsilon_0t}}{2(\epsilon_0^2 - \epsilon_*^2)}.
\end{align*}
Now if $\epsilon_0 = \epsilon_*$, then by similar computations we have
\[f_{\eta_i-\eta} (z) = (\frac{\epsilon_0}{4} + \frac{\epsilon_0^2\abs{z}}{4})e^{-\epsilon_0\abs{z}}  \] and 
\begin{align*}
    \PP(\eta_i - \eta \geq -t ) 
    &= 1 - (\frac{2+\epsilon_0t}{4})e^{-\epsilon_0t}.
\end{align*}

\end{appendix}
\fi

\end{document}